\documentclass{article}

\usepackage[final]{neurips_2025}

\usepackage[utf8]{inputenc} % allow utf-8 input
\usepackage[T1]{fontenc}    % use 8-bit T1 fonts
\usepackage{url}            % simple URL typesetting
\usepackage{booktabs}       % professional-quality tables
\usepackage{amsfonts}       % blackboard math symbols
\usepackage{nicefrac}       % compact symbols for 1/2, etc.
\usepackage{microtype}      % microtypography
\usepackage[table]{xcolor}         % colors
\usepackage{enumitem}       % customize list environments
\usepackage{amsmath}
\usepackage{amsthm}
\usepackage{amssymb}
\usepackage{mathtools}
\usepackage{graphicx}
\usepackage{mdframed}
\usepackage{url}
\usepackage{float}
\usepackage{caption}
\usepackage[toc,page,header]{appendix}
\usepackage{wrapfig}
\usepackage{algorithm,algorithmic}
\usepackage{enumitem}
\usepackage{notation}

\usepackage[linktocpage=true]{hyperref}
\hypersetup{
  colorlinks=true,
  citecolor=blue,
  urlcolor=cyan,
  linkcolor=magenta,
}

\usepackage{thmtools, thm-restate}
\declaretheorem[name=Theorem, numberwithin=section]{theorem}
\declaretheorem[name=Lemma, numberlike=theorem]{lemma}
\declaretheorem[name=Proposition, numberlike=theorem]{proposition}

\theoremstyle{definition}
\newtheorem*{main_result}{Informal Theorem}

\declaretheorem[name=Definition, numberlike=theorem]{definition}

\definecolor{light-gray}{gray}{0.8}

\renewcommand{\ss}{\smallskip}

\title{Optimism Without Regularization: \\
  Constant Regret in Zero-Sum Games}

\author{%
    John Lazarsfeld \\ SUTD\thanks{Correspondence to: \color{magenta}\texttt{jlazarsfeld@gmail.com}} \\ 
    \And
    Georgios Piliouras \\
    SUTD \\
    \And
    Ryann Sim \\ 
    SUTD \\ 
    \And
    Stratis Skoulakis \\
    Aarhus University \\
} 

\begin{document}
\maketitle

\begin{abstract}
  % Abstract

This paper studies the \textit{optimistic} variant
of Fictitious Play for learning in two-player zero-sum games.
While it is known that Optimistic FTRL
-- a \textit{regularized} algorithm 
with a bounded stepsize parameter -- 
obtains constant regret in this setting, we show for the 
first time that similar, optimal rates are also achievable
\textit{without} regularization: 
we prove for two-strategy games that Optimistic Fictitious Play
(using \textit{any} tiebreaking rule)
obtains only \textit{constant regret}, 
providing surprising new evidence on the ability of \textit{non}-no-regret
algorithms for fast learning in games.
Our proof technique leverages a geometric view
of Optimistic Fictitious Play in the dual space of payoff vectors,
where we show a certain energy function
of the iterates remains bounded over time. 
Additionally, we also prove a 
regret \textit{lower bound} of $\Omega(\sqrt{T})$ 
for \textit{Alternating} Fictitious Play. 
In the unregularized regime, this separates the ability
of optimism and alternation in achieving $o(\sqrt{T})$ regret.

%%% Local Variables:
%%% mode: latex
%%% TeX-master: "../ofp-neurips-main"
%%% End:

\end{abstract}

% SECTIONS
% Intro

\section{Introduction}
\label{sec:intro}

Despite the fact that regularization is
essential for no-regret online learning
in general adversarial settings, \textit{un}regularized
algorithms are still able to obtain \textit{sublinear regret}
in the case of two-player zero-sum games.
Fictitious Play (FP), dating back to~\cite{brown1951iterative},
is the canonical example of such an unregularized algorithm,
and it results from both players
independently running Follow-the-Leader (FTL).\footnote{
  FTL is a particular instance
  of Follow-the-Regularized-Leader (FTRL)
  with unbounded stepsize $\eta \to \infty$. 
}
In the worst case, FTL can have $\Omega(T)$ regret
due to its sensitivity to oscillations in adversarially-chosen
reward sequences~\citep{shalev2012online}.
However, in zero-sum game settings,
the classic result of 
\cite{robinson1951iterative}    
proved that, under Fictitious Play, the \textit{sum}
of the players' regrets
(henceforth referred to as \textit{regret}) is indeed sublinear
(thus implying time-average convergence to a Nash equilibrium),
albeit at the very slow $O(T^{1-1/n})$ rate
for $n\times n$ games (which was shown to be tight
by \citet{daskalakis2014counter} using an 
adversarial tiebreaking rule).

However, the recent works of~\citet{abernethy2021fast}
and~\citet{lazarsfeld2025fp} have established improved
$O(\sqrt{T})$ regret guarantees for Fictitious Play
on diagonal payoff matrices (using lexicographical tiebreaking)
and on generalized Rock-Paper-Scissors matrices
(using any tiebreaking rule), respectively. 
As a result, there is growing evidence on the robustness of
unregularized algorithms like FP
(\textit{not} a no-regret algorithm in general)
for obtaining fast, sublinear regret in zero-sum games.

On the other hand, the past decade has 
seen \textit{regularized} learning algorithms  
exhibit a remarkable success in providing even better
$o(\sqrt{T})$ regret guarantees for learning in games.
In zero-sum games,
\textit{optimistic} variants of FTRL
(including Optimistic Multiplicative Weights and
Optimistic Gradient Descent) 
obtain only \textit{constant} regret (with respect to
the time horizon $T$), implying
optimal $O(1/T)$ time-average convergence to Nash
\citep{rakhlin2013optimization, syrgkanis2015fast}.
While such guarantees can be obtained
using absolute constant stepsizes (with no dependence
on $T$), standard proof techniques (e.g., the
RVU bound approach of~\citet{syrgkanis2015fast})
still crucially require a \textit{finite upper bound
on stepsize}, corresponding to constant
magnitudes of regularization. 
This raises the following, fundamental question:
\textit{Is $O(1)$ regret attainable in zero-sum games
  without regularization
  (equivalently, with unbounded stepsizes)? 
  Can variants of Fictitious Play achieve $O(1)$ regret?}

Apart from its theoretical interest, this question admits 
crucial applications in the context of equilibrium computation
algorithms for combinatorial games~\citep{BeagleholeHKLL23},
as well as during training via self-play in certain multi-agent
reinforcement learning settings~\citep{vinyals2019grandmaster}.

\subsection{Our Contributions}
In this work, we establish an affirmative answer to the question above. 
Our main result establishes that, in the case of 2$\times$2 zero-sum games, 
\textit{Optimistic Fictitious Play} obtains constant regret: 
\begin{main_result}[See Theorem~\ref{thm:ofp-reg}]
\textit{Optimistic Fictitious Play, using any tiebreaking rule,
obtains $O(1)$ regret in all 2$\times$2 zero-sum games
with a unique, interior Nash equilibrium.}
\end{main_result}

Our result gives surprising
new evidence that, even without regularization, 
optimism can be used to obtain an accelerated regret bound,
matching the optimal rate obtained by regularized 
Optimistic FTRL algorithms \citep{syrgkanis2015fast}. 
While our theorem establishes constant regret 
only for the case of two-strategy zero-sum games,
our proof techniques offer indication that
similar, optimal regret bounds may further hold
in higher-dimensional settings.
Apart from our theoretical results, we also experimentally
evaluate Optimistic FP on higher-dimensional zero-sum games,
and these evaluations suggest that, even for much larger games,
Optimistic FP still obtains constant regret. 

Our proof technique is based on a novel geometric perspective
of Optimistic FP in the dual space of payoff vectors. 
Our main technical contribution is showing that an \textit{energy function} of
the dual iterates of the algorithm is upper bounded by a constant. 
This energy upper bound can then be easily used to establish constant regret of 
the primal iterates. The latter comes in contrast to the energy growth
of the iterates of standard FP, which strictly increases over
the time horizon~\citep{lazarsfeld2025fp}.

We also consider the \textit{alternating} variant
of Fictitious Play. Recent work has studied the use of alternation 
(independently of optimism) as a method for obtaining $o(\sqrt{T})$ 
regret guarantees in both the adversarial \citep{gidel2019negative,bailey2020finite,cevher2023alternation,
  hait2025alternating} and the zero-sum game setting
\citep{wibisono2022alternating,katona2024symplectic}. 
Contrary to \textit{optimism}, we show in the case
of \textit{alternation} that regularization is necessary 
to achieve $o(\sqrt{T})$ regret:
\begin{main_result}[See Theorem~\ref{thm:afp-regret-lb}]
\textit{
On the 2$\times$2 Matching Pennies game,
  Alternating Fictitious Play,
  using any tiebreaking rule and
  for nearly all initializations,
  has regret at least $\Omega(\sqrt{T})$.}
\end{main_result}

Together, our results separate the regret guarantees 
of using optimism and alternation in the regime of no regularization:
while optimism without regularization can obtain 
optimal $O(1)$ regret (Theorem~\ref{thm:ofp-reg}),
alternation alone is in general insufficient for improving
beyond $O(\sqrt{T})$ (Theorem~\ref{thm:afp-regret-lb}), the same rate
achievable by standard (non-alternating) FP
in the 2$\times$2 setting.
To this latter point, note that the lower bound 
of Theorem~\ref{thm:afp-regret-lb} comes in contrast 
to the improved $O(T^{1/5})$ rate obtainable by 
Alternating FTRL under a sufficiently small stepsize
\citep{katona2024symplectic}.

Table~\ref{table:results} summarizes our results and
the landscape of regret guarantees in the
2$\times$2 setting for FTRL and FP variants,
and Figure~\ref{fig:regret-comp-intro}
shows an example of the empirical regret guarantees of 
standard, Optimistic, and Alternating FP
variants in several games 
(additional results are presented in 
Sections~\ref{sec:conclusion} and~\ref{app:experiments}).

\begin{table}[bh!]
  \centering
  \renewcommand{\arraystretch}{1.6}
  \begin{tabular}{c | c | c | c }
    & \textbf{Standard}
    & \textbf{Optimistic}
    & \textbf{Alternating}  \\ \hline 
    $\eta$ bounded (FTRL)
    & $O(\sqrt{T})$ $\dagger$
    & $O(1)$ $^\wedge$
    & $O(T^{1/5})$ $^{\wedge\wedge}$ \\ \hline    
    $\eta \to \infty$ (FP)
    & $O(\sqrt{T})$ $^\ddagger$
    & \cellcolor{light-gray}$O(1)$ $^\star$
    & \cellcolor{light-gray}$\Omega(\sqrt{T})$  $^{\star\star}$
  \end{tabular}
  \vspace*{0.5em}
  \caption{%
    \footnotesize
    Regret guarantees for
    FTRL and Fictitious Play variants in
    2$\times$2 zero-sum games, with our 
    contributions shaded in gray.
    $\dagger$: Using the standard setting of 
    $\eta = 1/\sqrt{T}$~\citep{shalev2012online}.
    $^\wedge$: Via the RVU bounds of~\cite{syrgkanis2015fast}.
    $^{\wedge\wedge}$: \citep{katona2024symplectic},
    extending on the prior $O(T^{1/3})$ bound of
    \cite{wibisono2022alternating}. 
    $\ddagger$: Implicit in the proof 
    of~\cite{robinson1951iterative}.
    $^\star$: Theorem~\ref{thm:ofp-reg}. 
    $^{\star\star}$: Theorem~\ref{thm:afp-regret-lb}.}
  \vspace*{-1em}
  \label{table:results}
\end{table}

\begin{figure}[t!]
\centering
\includegraphics[width=0.85\textwidth]{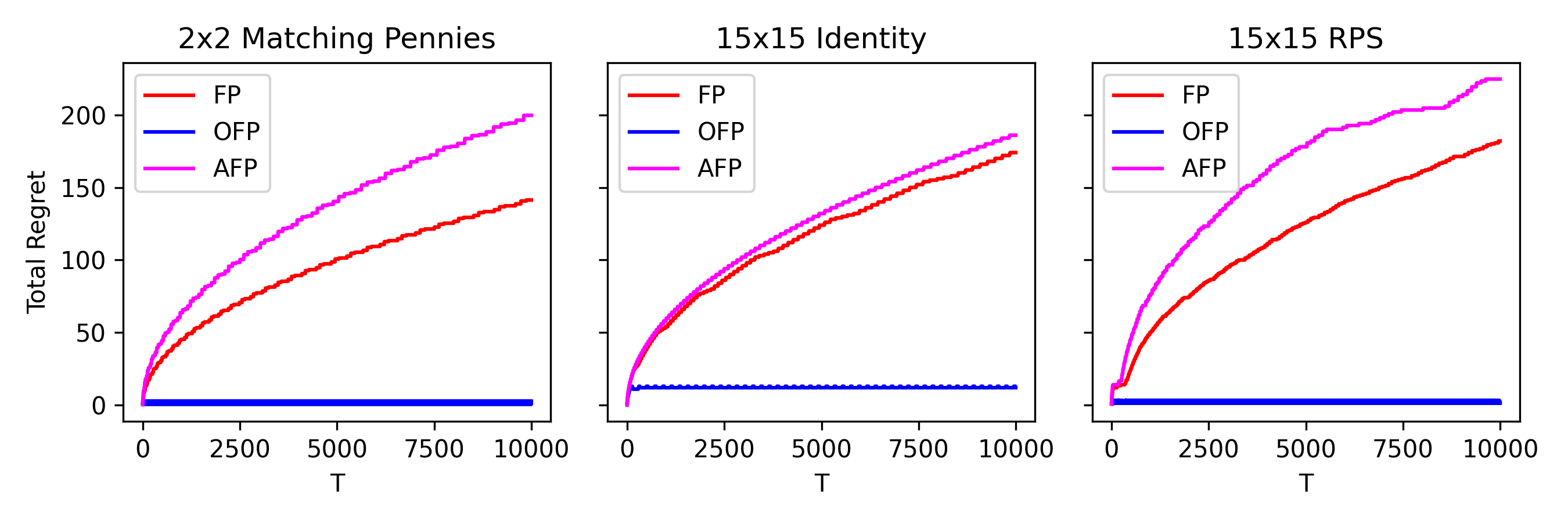}
\vspace*{-1em}
\caption{\small
    Empirical regret of standard (FP), Optimistic (OFP),
    and Alternating (AFP) Fictitious Play 
    in Matching Pennies 
    (from $x^0_1 = (1/3, 2/3)$, $x^0_2 = (2/3, 1/3)$),
    on the 15$\times$15 identity matrix 
    (from $x^0_1 = e_1$, $x^0_2 = e_n$),
    and on 15$\times$15 generalized Rock-Paper-Scissors
    (from $x^0_1 = e_1$, $x^0_2 = e_n$).
    Each algorithm was run for $T=10000$ iterations
    using a lexicographical tiebreaking rule. 
    Each subfigure demonstrates the 
    constant empirical regret of OFP compared
    to the roughly $\sqrt{T}$ regret growth 
    of standard FP and AFP. 
    More experimental details and results are given in 
    Section~\ref{sec:conclusion} and Section~\ref{app:experiments}.
    }
\vspace*{-1em}
\label{fig:regret-comp-intro}
\end{figure}

\subsection{Other Related Work}

\noindent\textbf{Optimistic learning in games.}
% \paragraph{Optimistic learning in games.}
Our work relates to a line of research on 
the convergence properties of optimistic and 
extragradient-type algorithms in both normal-form 
and extensive-form games~\citep{DFG21,FMPV22,anagnostides2022near,
ZAFS24,FALLKS22,HACM22,AFKL22,
piliouras2022beyond,APFS22}. 
As previously noted, a key difference is that 
these approaches typically rely on constant or time-decreasing step sizes,
corresponding to some level of regularization.
Recent works have also investigated last-iterate
convergence properties of optimistic and extragradient 
methods~\citep{DaskalakisP19,CZ22,bernethyLW21,HIMM20}, 
as well as for variants of regret-matching, including under
alternation~\citep{FGKLLZ25}.
Other works have studied accelerated rates
using optimism \textit{without} regularization
in certain Frank-Wolfe-type, convex-concave saddle-point problems
\citep{WA07,ALLW18}.

\noindent\textbf{Learning in 2$\times$2 games.}
% \paragraph{Learning in 2$\times$2 games.}
Our work adds to a growing recent literature studying 
online learning algorithms in 2$\times$2 games:
in two-strategy zero-sum games, \cite{bailey2019fast} 
proved that Online Gradient Descent obtains $O(\sqrt{T})$ 
regret even with large constant stepsizes. 
More recent works of \cite{cai2024fast} and~\cite{C25}
establish lower-bounds on the 
last-iterate and random-iterate convergence rates of
Optimistic MWU using a hard 2$\times$2 construction, as well as 
an $O(T^{1/6})$ upper bound on 
best-iterate convergence for 2$\times$2 zero-sum games. 
\cite{chen2020hedging} similarly used a 2$\times$2 construction
to establish a general $\Omega(\sqrt{T})$ lower bound on 
the regret of standard Multiplicative Weights.
In a family of two-strategy congestion games,
\cite{chotibut2021family} also showed that the iterates of
Multiplicative Weights can exhibit formally chaotic behavior. 

\noindent\textbf{Fictitious Play.}
% \paragraph{Fictitious Play.}
We refer the reader to the recent results of 
\cite{daskalakis2014counter}, \cite{abernethy2021fast},
and \cite{lazarsfeld2025fp}
(and the references therein) for background on standard Fictitious Play.
\cite{abernethy2021fast} also briefly introduced the optimistic variant
of Fictitious Play, and they informally conjecture the algorithm
to have constant regret on diagonal payoff matrices. 
The convergence behavior of FP has also been studied
in potential games~\citep{monderer1996potential, PPSC23}, 
near-potential games~\citep{candogan}, 
Markov games~\citep{SayinPO22,BaudinL22}, and extensive-form 
games~\citep{heinrich2015fictitious}.

%%% Local Variables:
%%% mode: latex
%%% TeX-master: "../ofp-neurips-main"
%%% End:
 
% Prelims
\section{Preliminaries}
\label{sec:prelims}

Let $[n] = \{1, \dots, n\}$,
let $\Delta_n$ be the probability
simplex in $\R^n$,
and let $\{e_i\}_n = \{e_i : i \in [n]\} \subset \Delta_n$
denote the set of standard basis vectors in $\R^n$,
which correspond to vertices of $\Delta_n$.
For $x \in \Delta_n$,
we say $x$ is \textit{interior} 
if $x_i > 0$ for all $i \in [n]$.

\subsection{Online Learning in Zero-Sum Games}
Let $A \in \R^{m \times n}$ be the payoff matrix
for a two-player zero-sum game,
and let $T$ be a fixed time horizon. 
At round $t$, Players 1 and 2
simultaneously choose mixed strategies
$x^t_1 \in\Delta_m$ and $x^t_2 \in \Delta_n$,
obtain payoffs $\langle x^t_1, A x^t_2 \rangle$
and $-\langle x^t_2, A^\top x^t_1 \rangle$,
and observe feedback
$Ax^t_2$ and $-A^\top x^t_1$, respectively.

\paragraph{Regret and Convergence to Nash.}
Each player seeks to maximize their cumulative payoff,
and their performance is measured by the 
individual regrets
$\reg_1(T)
= \max_{x \in \Delta_m}
\sum_{t=1}^T \langle x - x^t_1, Ax^t_2\rangle$
and
$\reg_2(T)
= \min_{x \in \Delta_n}
\sum_{t=1}^T \langle x^t_2 - x, A^\top x^t_1\rangle$.
From a global perspective, we study the
\textit{total regret} (henceforth \textit{regret})
$\reg(T) = \reg_1(T) + \reg_2(T)$ given by
\begin{equation}
\reg(T)
\;=\;
  \max_{x \in \Delta_m}\;
  \big\langle x, \sum\nolimits_{t=0}^T Ax^t_2 \big\rangle
  -
  \min_{x \in \Delta_n}\;
  \big\langle x, \sum\nolimits_{t=0}^T A^\top x^t_1\big\rangle \;.
  \label{eq:regret}
\end{equation}
It is well known that sublinear bounds
on $\reg(T)$ correspond to convergence (in duality gap) of
the players' time-average iterates to a Nash
equilibrium (NE) of $A$.
Recall that the duality gap of a joint strategy profile 
$(x_1, x_2)$ is given by
$\DG(x_1, x_2) = \max_{x'_1 \in \Delta_m}
\langle x'_1, Ax_2\rangle
- \min_{x'_2 \in \Delta_n}
\langle x'_2, A^\top x_1 \rangle$,
and that $(x^*_1, x^*_2)$ is an NE of $A$
if and only if $\DG(x^*_1, x^*_2) = 0$.
Then the following relationship holds
(see Section~\ref{app:prelims} for a proof):

\begin{restatable}{proposition}{propregretnash}
  \label{prop:regret-nash}
  Let $\widetilde x^T_1 = \frac{1}{T}(\sum_{t=0}^T x^t_1)$
  and $\widetilde x^T_2 = \frac{1}{T}(\sum_{t=0}^T x^t_2)$
  denote the time-average iterates of Players 1 and 2,
  respectively, and suppose $\reg(T) = o(T)$.
  Then $(\widetilde x^T_1, \widetilde x^T_2)$
  converges (in duality-gap) to an NE of $A$
  at a rate of $\reg(T)/T = o(1)$. 
\end{restatable}

% Note that if $\reg(T) \le O(1)$,
% the time-averaged convergence to Nash has
% an optimal $O(1/T)$ rate.

\subsection{Fictitious Play and Optimistic Fictitious Play}

We now introduce the Optimistic Fictitious Play
(and standard Fictitious Play) algorithms.
The primal update rules for both standard and
Optimistic Fictitious Play can
be described via the following
$\alpha$-Optimistic Fictitious Play ($\alpha$-OFP) expression:
\begin{equation*}
  \begin{aligned}  
    x^{t+1}_1
    &\;:=\;
      \argmax_{x \in \{e_i\}_m}
      \big\langle
      x, \sum\nolimits_{k=0}^{t}Ax^k_2 + \alpha Ax^t_2
      \big\rangle \\
    x^{t+1}_2
    &\;:=\;
      \argmax_{x \in \{e_i\}_n}
      \big\langle
      x, \sum\nolimits_{k=0}^{t}-A^\top x^k_1 - \alpha A^\top x^t_1
      \big\rangle \;.
  \end{aligned}
  \label{eq:alpha-ofp}
  \tag{$\alpha$-OFP}
\end{equation*}
When $\alpha=0$, then~\eqref{eq:alpha-ofp}
recovers standard Fictitious Play, where each
player's strategy at time $t+1$ is a best response
to the sum of its feedback vectors through round $t$.
Optimistic Fictitious Play (OFP) is
the setting of~\eqref{eq:alpha-ofp} with $\alpha =1$. 
Observe this recovers
the \textit{unregularized variant} of Optimistic FTRL
(equivalently, with $\eta \to \infty$)
in the zero-sum game setting
(c.f., \citep{rakhlin2013optimization,syrgkanis2015fast}),
which adds bias to the most recent feedback vector. 

\begin{restatable}[\bf Tiebreaking Rules]{remark}{remarktiebreak}
  \label{remark:tiebreak}
  Observe that the $\argmax$ sets in~\eqref{eq:alpha-ofp}
  may contain multiple vertices.
  For this, we assume that the $\argmax$ operator
  encodes a \textit{tiebreaking rule} that returns
  a distinct element. Throughout, we make \textit{no assumptions
  on the nature of the tiebreaking rule},
  and in general ties can be broken deterministically,
  randomly, or adaptively/adversarially. 
\end{restatable}

\paragraph{Dual payoff vectors and primal-dual update.}
Optimistic FP can be equivalently
written with respect to the \textit{cumulative}
payoff vectors $y^t_1 = \sum_{k=0}^{t-1} Ax^k_2 \in \R^m$
and $y^t_2 = \sum_{k=0}^{t-1} -A^\top x^k_1 \in \R^n$.
Specifically, the iterates of the algorithm can be expressed
in the following \textit{primal-dual} form:
\begin{restatable}{definition}{defofppd}
  \label{def:ofp-pd}
  Let $y^0_1 = 0 \in \R^m$ and $y^0_2 = 0 \in \R^n$,
  and fix any initial $x^0_1 \in \Delta_m$
  and $x^0_2 \in \Delta_n$. Then for $t \ge 1$,
  the dual (i.e., $(y^t_1, y^t_2)$)
  and primal (i.e., $(x^t_1, x^t_2)$) iterates
  of Optimistic FP are:
  \begin{equation}
    \begin{cases}
      \;y^t_1
      \;=\;
      y^{t-1}_1 + Ax^{t-1}_2 \\
      \;y^t_2
      \;=\;
      y^{t-1}_2 - A^\top x^{t-1}_1 
    \end{cases}
    \quad\text{and}\quad
    \begin{cases}
      x^t_1
      \;=\;
      \argmax\limits_{x \in \{e_i\}_m}\;
      \big\langle x, 
      y^t_1 + Ax^{t-1}_2
      \big\rangle \\
      x^{t}_2
      \;=\;
      \argmax\limits_{x \in \{e_i\}_n}\;
      \big\langle x, 
      y^t_2 - A^\top x^{t-1}_1
      \big\rangle 
    \end{cases}
    \;.
    \label{eq:ofp}
    \tag{OFP}
  \end{equation}
\end{restatable}

%%% Local Variables:
%%% mode: latex
%%% TeX-master: "../ofp-neurips-main"
%%% End:

% main result

\section{Regret Bounds for Optimistic and Alternating Fictitious Play}
\label{sec:main-result}

The main result of this paper establishes
a \textit{constant} regret bound for Optimistic Fictitious Play
in two-strategy zero-sum games.
Formally we prove the following theorem:

\begin{restatable}{theorem}{thmofpreg}
  \label{thm:ofp-reg}
  Let $A$ be a 2x2 zero-sum game
  with a unique interior NE,
  and let $\{(x^t_1, x^t_2)\}$ be the
  iterates of~\eqref{eq:ofp} on $A$
  using any tiebreaking rule. 
  Then $\reg(T) \le O(1)$. 
\end{restatable}

As mentioned, this result establishes the
first constant regret bounds for 
Fictitious Play variants in the two-player zero-sum game setting,
and the result holds regardless of the tiebreaking rule used. 
Moreover, this bound matches the optimal 
rate obtained by Optimistic FTRL variants 
for zero-sum games \citep{syrgkanis2015fast},
but notably, our proof technique departs significantly from 
the RVU bound approach used to obtain those results.

As a consequence of the techniques we develop
for proving Theorem~\ref{thm:ofp-reg}, we also
establish a \textit{lower bound} on the regret
of \textit{Alternating} Fictitious Play.
In particular, for the corresponding
\textit{alternating} regret $\regalt(T)$
(see Definition~\ref{def:alt-regret}), we prove
on the 2$\times$2 Matching Pennies game the
following:

\begin{restatable}{theorem}{thmafpreglb}
  \label{thm:afp-regret-lb}
  Suppose $x^1_1 = (p, 1-p) \in \Delta_2$
  for irrational $p \in (3/4, 1)$, and let $\{x^t\}$
  be the iterates of Alternating FP on~\eqref{eq:MP}
  using any tiebreaking rule.
  Then $\regalt(T) \ge \Omega(\sqrt{T})$.
\end{restatable}

To streamline the presentation of the paper,
we defer the precise descriptions of
the alternating play model, alternating regret,
and the Alternating FP algorithm to
Section~\ref{app:afp-lb}, where we also
develop the proof of Theorem~\ref{thm:afp-regret-lb}.
Instead, in the remainder of the main text,
we focus on developing the
proof of Threorem~\ref{thm:ofp-reg}.
To this end, we proceed to give an overview of
our techniques.

\subsection{Intuition and Overview of Proof Techniques
for Theorem~\ref{thm:ofp-reg}}

To prove Theorem~\ref{thm:ofp-reg},
we leverage a geometric view of Optimistic Fictitious Play
in the dual space of payoff vectors.
We give a brief overview of this geometric
perspective here:

\paragraph{Energy and regret.}
First, we show that the regret of
the primal iterates $\{x^t\}$
is equivalent to the growth of
an \textit{energy function} of the dual iterates $\{y^t\}$.
Specifically, define the energy $\Psi$ as follows:
\begin{restatable}{definition}{defPsi}
  \label{def:Psi}
  Let $y^t := (y^t_1, y^t_2)$ be the concatenated
  primal and dual iterates of~\eqref{eq:ofp}
  at time $t \ge 1$.
  Then for $y = (y_1, y_2) \in \R^{m+n}$,
  the energy function
  $\Psi: \R^{m+n} \to \R$ is given by
  \begin{equation}
    \Psi(y)
    \;=\;
    \max_{x \in \Delta_m \times \Delta_n}\;
    \langle x, y \rangle  \;.
    \label{eq:energy-full}
  \end{equation}
\end{restatable}
In other words, $\Psi$ is the \textit{support function}
of $\Delta_m \times \Delta_n$.
Then by definition of $\reg(T)$ and the
payoff vectors $\{y^t\}$,
the following relationship holds
(see Section~\ref{app:prelims} for a proof):

\begin{restatable}{proposition}{propenergyregret}
  \label{prop:energy-regret}
  Let $\{x^t\}$ and $\{y^t\}$
  be the iterates of~\eqref{eq:ofp}.
  Then $\reg(T) = \Psi(y^{T+1})$. 
\end{restatable}

Due to Proposition~\ref{prop:energy-regret},
it is immediate that a constant upper bound on
$\Psi(y^{T+1})$ implies a constant upper bound
on $\reg(T)$. To this end, our main technical
contribution is to prove the following
upper bound on the energy $\Psi$ under 
Optimistic Fictitious Play:

\begin{restatable}{theorem}{thm:ofp-energy-2x2}
  \label{thm:ofp-energy-2x2}
  Assume the setting of Theorem~\ref{thm:ofp-reg}.
  Let $\amax = \|A\|_\infty$ denote the largest
  entry of $A$, and let $\agap = \min_{(i, j), (k, \ell)}
  |A_{ij} - A_{k\ell}|$ denote the
  smallest absolute difference between two entries of $A$.
  Let $\{y^t\}$ denote the
  dual iterates of~\eqref{eq:ofp} on $A$.
  Then
  $\Psi(y^{T+1})
  \le
  8 \amax\big(1 + 2\big(\frac{\amax}{\agap}\big)\big)^2$.
\end{restatable}
In other words, the energy of the dual iterates
under Optimistic FP are bounded by an absolute constant
that depends only on the entries of $A$.
In Section~\ref{sec:ofp-regret}, we give
a technical overview of the proof of
Theorem~\ref{thm:ofp-energy-2x2},
but we first present more introduction and intuition
on the geometric perspective of
Optimistic Fictitious Play that is used to
prove the result. 

\paragraph{Fictitious Play as Skew-Gradient Descent.}
As shown in~\citet{abernethy2021fast} and~\citet{lazarsfeld2025fp},
in the dual space of payoff vectors, 
standard Fictitious Play can be viewed as a certain
\textit{skew-gradient descent} with respect to the energy $\Psi$. 
In light of this,  we introduce a common geometric viewpoint 
that captures both standard and Optimistic FP and gives
insight into their differences in energy growth,
as implied by Theorem~\ref{thm:ofp-energy-2x2}.

For this, note that the
\textit{subgradient set} of $\Psi$ at $y \in \R^{m + n}$ is given by
$
\partial \Psi(y)
  =
  \argmax_{x \in \Delta_m \times \Delta_n}
\langle x, y\rangle 
$.
Then both standard FP and Optimistic FP can be expressed as a
\textit{skew-(sub)gradient-descent}
with respect to $\Psi$ evaluated at a
\textit{predicted} dual vector
$\widetilde y^{t+1}$
(see Section~\ref{app:prelims:energy:sgd} for
a full derivation):
\begin{restatable}{proposition}{propskewgradient}
  \label{prop:skew-gradient}
  Let $\{y^t\}$ denote the dual iterates of
  either standard Fictitious Play
  (e.g.,~\eqref{eq:alpha-ofp} with $\alpha=0$)
  or Optimistic Fictitious Play.
  Then for all $t \ge 1$, the iterates of
  each algorithm evolve as
  \begin{equation}
    \setlength\arraycolsep{1pt}
    \begin{cases}
      \;y^t
      \;=\;
      y^{t-1} + J x^{t-1} \\
      \;x^t
      \;\in\;
      \partial \Psi\Big(
      \widetilde y^{t+1}
      \Big)
    \end{cases}
    \text{where}\;
    J =
    \begin{pmatrix}
      0 & A \\
      -A^\top & 0 
    \end{pmatrix}
    \;\;\text{and}\;\;
    \widetilde y^{t+1}
    =
    \begin{cases}
      \;y^t &\text{for FP} \\
      \;2y^t - y^{t-1}
            &\text{for OFP} 
    \end{cases} \;,
    \label{eq:fp-pd-variants}
  \end{equation}
  and it follows inductively that
  $
  y^{t+1}
  =\;
  y^t + J \partial \Psi(\widetilde y^{t+1})
  $,
  where $\partial\Psi(\widetilde y^{t+1})$ denotes
  a fixed vector in the subgradient set
  of $\Psi$ at $\widetilde y^{t+1}$.
\end{restatable}

\paragraph{One-step energy growth comparison of
Fictitious Play variants.}
For standard FP, due to its Hamiltonian structure,
the analogous skew-gradient flow in continuous-time is
known to exactly \textit{conserve} $\Psi$, leading to constant regret
(see, e.g.,~\cite{mertikopoulos2018cycles, abernethy2021fast,
  wibisono2022alternating}).
However, due to discretization, this energy conservation
does not hold in general under discrete-time Fictitious Play variants.
For example, under each step of standard Fictitious Play, $\Psi$
is always non-decreasing. To see this, let
$\Delta \Psi(y^t) = \Psi(y^{t+1}) - \Psi(y^t)$, 
and by slight abuse of notation, let
$\partial \Psi(y)$ denote a fixed vector in the subgradient 
set of $\Psi$ at $y \in \R^{m+n}$. 
Then by Jensen's inequality and skew-symmetry of $J = -J^\top$,
it holds for all $t\ge 1$ that
\begin{equation}
  \textbf{For FP:}\quad
  \Delta\Psi(y^t)
  \;\ge\;
    \big\langle
    \partial \Psi(y^t), J \partial \Psi(\widetilde y^{t+1})
    \big\rangle
    =
    \big\langle
    \partial \Psi(y^t), J \partial \Psi(y^t)
    \big\rangle
    = 0
    \;.
    \label{eq:fp-energy-step}
\end{equation}
In fact, the recent upper bounds of~\citet{abernethy2021fast} and
\cite{lazarsfeld2025fp} imply that under FP, 
$\Psi$ is \textit{strictly} increasing by a constant in
roughly $\sqrt{T}$ iterations.

On the other hand, for Optimistic FP 
using $\widetilde y^{t+1} = 2y^t - y^{t-1}$,
we instead have by Jensen's inequality:
\begin{equation}
  \textbf{For Optimistic FP:}\quad
  \Delta\Psi(y^t)
  \le
    \big\langle
    \partial \Psi(y^{t+1}), J \partial \Psi(\widetilde y^{t+1})
    \big\rangle
    =
    \big\langle
    \partial \Psi(y^{t+1}), J \partial \Psi(\widetilde y^{t+1})
    \big\rangle.
    \label{eq:ofp-energy-step}
\end{equation}
Thus by skew-symmetry of $J$,
expression~\eqref{eq:ofp-energy-step} reveals
that for any $t$ where
$\partial \Psi(y^{t+1}) = \partial \Psi(\widetilde y^{t+1})$
(e.g., true dual vector $y^{t+1}$
and predicted dual vector $\widetilde y^{t+1}$
both ``map'' to the same primal vertex), then 
the one-step energy growth $\Delta\Psi(y^t) \le 0$
is \textit{non-increasing} under Optimistic FP. 

\paragraph{Challenges in establishing
  non-positive energy growth for OFP.}
Naively, one might in general hope
the invariant $\partial \Psi(y^{t+1}) = \partial \Psi(\widetilde y^{t+1})$
holds at \textit{every} timestep under Optimistic FP.
However, simple experiments reveal that this is not true:
in general $\Psi$ \textit{can increase} during one step of the algorithm. 
Thus, understanding when and why such an invariant
\textit{does} hold is still a challenging task that
may require leveraging structural properties of the payoff matrix.
In the proof of Theorem~\ref{thm:ofp-energy-2x2},
we leverage such properties of 2$\times$2 games 
to establish sufficient conditions for when the above invariant
holds, and this subsequently leads to a constant upper bound on energy.

%%% Local Variables:
%%% mode: latex
%%% TeX-master: "../ofp-neurips-main"
%%% End:

\section{Bounded Energy Under Optimistic Fictitious Play}
\label{sec:ofp-regret}

In this section, we now give a technical overview
of the proof of Theorem~\ref{thm:ofp-energy-2x2}.
Throughout the proof, we make
the following assumptions on the payoff matrix $A$:
\begin{restatable}{assumption}{assA}
  \label{ass:2x2}
  Let $A \in \R^{2\times 2}$. Assume that
  \begin{equation*}
    A \;=\;
    \begin{pmatrix}
      a & b \\
      c & d
    \end{pmatrix}
    \quad\text{where}\quad
    \begin{cases}
      \;\;\text{(i)}&\det A = ad - bc = 0 \\
      \;\text{(ii)}&a, d \;>\; \max\{0, b, c\}
    \end{cases} \;\;.
  \end{equation*}
\end{restatable}
As proven by \citet{bailey2019fast},
who studied online gradient descent in 2$\times$2 games,
for any Fictitious Play or FTRL variant,
Assumption~\ref{ass:2x2} holds without loss of generality:
\begin{restatable}[\cite{bailey2019fast}]{proposition}{propAwlog}
  \label{prop:A-ass-wlog}
  Let $A \in \R^{2\times2}$ have a unique, interior NE,
  and let $\{x^t\}$ be the
  iterates of~\eqref{eq:ofp} on $A$.
  Then there exists $\widetilde A \in \R^{2\times2}$
  satisfying Assumption~\ref{ass:2x2} such that
  (1)  $\widetilde A$ and $A$ have the same NE and
  (2) the iterates
  $\{\widetilde x^t\}$
  of running~\eqref{eq:ofp} on $\widetilde A$
  are identical to $\{x^t\}$.
\end{restatable}
For completeness, we include a full proof of this result
in Proposition~\ref{prop:A-ass-wlog} of Section~\ref{app:wlog}.
The key consequence of the assumption is that,
under Optimistic FP, the dual payoff vectors
$y^t = (y^t_1, y^t_2) \in \R^4$ all
lie in the \textit{same two-dimensional subspace}.
Formally, we have:

\begin{restatable}{proposition}{propsubspace}
  \label{prop:subspace}
  Let $A$ satisfy Assumption~\ref{ass:2x2},
  and let $\{y^t_1\}$ and $\{y^t_2\}$ be the
  dual payoff vectors of~\eqref{eq:ofp}.
  Then for every $t \ge 1$,
  it holds that
  $y^t_{12} = -\rho_1 \cdot y^t_{11}$
  and $y^t_{22} = - \rho_2 \cdot y^t_{21}$,
  where $\rho_1 := (d-c)/(a-b) > 0$ and $\rho_2 = (d-b)/(a-c) > 0$.
\end{restatable}
The proof of Proposition~\ref{prop:subspace} is given in
Section~\ref{app:ofp-regret}.
Importantly, as $\rho_1, \rho_2 > 0$,
observe that 
$y^t_{11} > 0 \iff y^t_{11} > y^t_{12}$, and
$y^t_{21} > 0 \iff y^t_{21} > y^t_{22}$
for all times $t \ge 1$. 
Thus, in the 2$\times$2 setting, the coordinates $y^t_{11}$
and $y^t_{21}$ encode all information needed to analyze
the iterates of Optimistic FP in~\eqref{eq:fp-pd-variants}.

With this in mind, the strategy for proving
the upper bound on the energy $\Psi(y^{T+1})$ is
as follows: first, leveraging the
observations above,
and as in \cite{bailey2019fast},
we restrict our study of the dual iterates
$(y^t_1, y^t_2) \in \R^4$ to the pair of 
scalar iterates $(y^t_{11}, y^t_{21}) \in \R^2$.
For this, we introduce in
Section~\ref{sec:ofp-regret:subspace-dynamics}
a new set of notation to capture this
lower-dimensional \textit{subspace dynamics},
which also naturally leads to the definition of an
\textit{equivalent energy function}. 
For the new, equivalent energy function,
we then prove in
Section~\ref{sec:ofp-regret:bounded-energy}
a set of invariants that allow for establishing
a uniform, constant upper bound on the
energy of the dual iterates over time. 

\subsection{Subspace Dynamics of Optimistic Fictitious Play}
\label{sec:ofp-regret:subspace-dynamics}

We now introduce an equivalent set
of primal and dual iterates
$\{w^t\}$ and $\{z^t\}$ for~\eqref{eq:ofp},
as well as a new, equivalent energy function $\psi$. 
We will establish in 
Proposition~\ref{prop:energy-equiv}
that $\Psi(y^{T+1}) = \psi(z^{t+1})$. 

\paragraph{Primal variables.}
First, by definition of~\eqref{eq:ofp},
both $x^t_1$ and $x^t_2$ are vertices of $\Delta_2$.
Letting
$\calX =
\{(1, 0), (0, 1)\} \times \{(1, 0), (0, 1)\} \subset \R^4$
denote the vertices of the
joint simplex $\Delta_2 \times \Delta_2$,
it follows for each $t \ge 1$ that $x^t \in \calX$. 
We define new primal iterates $w^t \in \R^4$,
where each $w^t$ is a standard basis vector of $\R^4$.
Let $\calW = \{e_1, e_2, e_3, e_4\} \subset \R^4$ denote this set.
Then for $t \ge 1$, let:
\begin{equation}
  w^t
  \;=\;
  \begin{cases}
    \;e_2 \iff x^t=(0, 1, 1, 0)
    \qquad
    e_3 \iff x^t=(1, 0, 1, 0) \\
    \;e_1 \iff x^t=(0, 1, 0, 1)
    \qquad
    e_4 \iff x^t=(1, 0, 0, 1) 
  \end{cases}\;.
  \label{eq:wt}
\end{equation}

\paragraph{Dual variables.}
For each $t \ge 0$, let $z^t_1 = y^t_{11} \in \R$
and $z^t_2 = y^t_{21} \in \R$.
Let $z^t = (z^t_1, z^t_2) \in \R^2$. 

\paragraph{Primal-dual map.}
To aid in the description and analysis of the subspace dynamics,
we describe the following partition of the dual space $\R^2$.
We then describe a corresponding choice map $\Q: \R^2 \to \calX$
that relates the primal and dual variables $\{w^t\}$ and $\{z^t\}$.

\medskip

\begin{definition}
  \label{def:P-sets}
  First, let $\calP = \{P_1, P_2, P_3, P_4\} \subset \R^2$,
  where each $P_i$ is the set
  \begin{equation}
    \begin{aligned}
      P_2 &\;=\;
            \{z \in \R^2: z_1 < 0\;\text{and}\; z_2 > 0 \}\qquad
            P_3 \;=\;
            \{z \in \R^2: z_1 > 0\;\text{and}\; z_2 > 0 \} \\
      P_1 &\;=\;
            \{z \in \R^2: z_1 < 0\;\text{and}\; z_2 < 0 \}\qquad
            P_4 \;=\;
            \{z \in \R^2: z_1 > 0\;\text{and}\; z_2 < 0 \} \;.
    \end{aligned}
    \label{eq:P-sets}
  \end{equation}
  Next, let
  $\widetilde \calP =
  \{P_{1\sim2}, P_{2\sim 3}, P_{3 \sim 4}, P_{4\sim 1}\}
  \subset \R^2$,
  where we define
  \begin{equation}
    \small
    \begin{aligned}
      P_{2\sim 3}
      &\;=\; \{z \in \R^2 : z_1 = 0 \;\text{and}\; z_2 > 0 \}
        \qquad
        P_{3\sim 4}
        \;=\; \{z \in \R^2 : z_1 >  0 \;\text{and}\; z_2 = 0 \} \\
      P_{1\sim 2}
      &\;=\; \{z \in \R^2 : z_1 < 0 \;\text{and}\; z_2 = 0 \}
        \qquad
        P_{4\sim 1}
        \;=\; \{z \in \R^2 : z_1 =  0 \;\text{and}\; z_2 < 0 \} \;.
    \end{aligned}
    \label{eq:P-boundary-sets}
  \end{equation}
  Finally let $\widehat \calP = \cup_{i\in [4]} \widehat P_i$,
  where $\widehat P_i = P_i \cup P_{i \sim (i+1)}$. 
  Observe by definition that $\widehat \calP \cup \{(0, 0)\} = \R^2$. 
\end{definition}

\begin{wrapfigure}{r}{0.5\textwidth}
  \centering
  \vspace*{-1em}                 
  \includegraphics[width=0.35\textwidth]{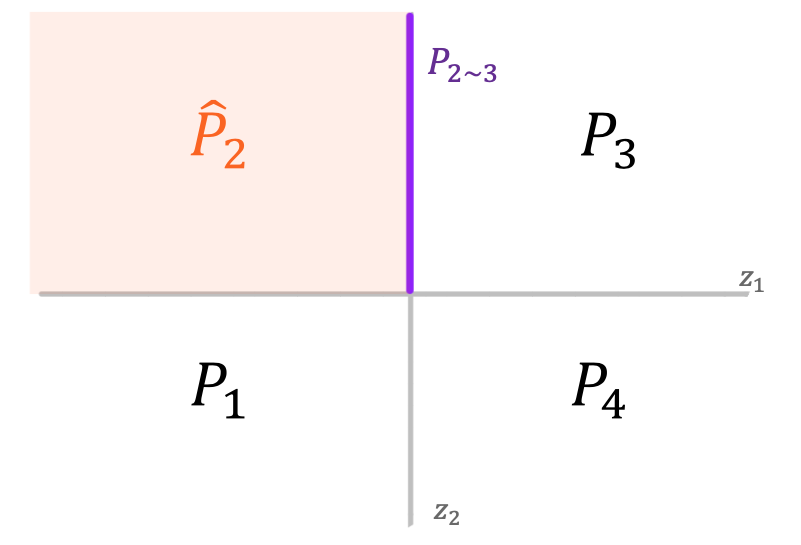}
  \captionof{figure}{\small
  Examples of the sets in
    $\calP$, $\widetilde \calP$, and $\widehat \calP$.}
  \vspace*{-1em}
  \label{fig:P-sets}
\end{wrapfigure}

Note that for notational convenience, when using an index $i \in [4]$
in the context of the sets of Definition~\ref{def:P-sets},
we assume addition and subtraction to $i$ are performed
$(\mod 4)$ in the natural way that maps to the set $\{1, 2, 3, 4\}$.
For example, $P_{i \sim (i+1)}$ is the set $P_{4 \sim 1}$
when $i = 4$, and $P_{i-2} = P_{i+2}$ is the set $P_3$ when $i = 1$, etc.
Figure~\ref{fig:P-sets} depicts the sets
from Definition~\ref{def:P-sets}.

Moreover, Definition~\ref{def:P-sets}
allows for defining a choice map $\Q: \R^2 \to \calW$
that encodes the primal update rule of~\eqref{eq:ofp}.
In particular, $\Q$ maps dual variables $z$
to primal vertices $w$ depending on the
membership of $z$ in $\calP$ or $\widetilde \calP$.
Formally:

\begin{definition}
  \label{def:Q-map}
  Let $\Q: \R^2 \to \calW$ be the map
  defined as follows: first, let $\Q((0, 0)) = e_1$. Then
  \begin{itemize}[
    leftmargin=2em,
    topsep=0em,
    itemsep=0em
    ]
  \item
    For $z \in \calP$: if $z \in P_i$ for $i \in [4]$,
    then $\Q(z) = e_i$. 
  \item
    For $z \in \widetilde \calP$:
    if $z \in P_{i\sim(i+1)}$ for $i \in [4]$,
    then $\Q(z) \in \{e_i, e_{i+1}\}$.
  \end{itemize}
\end{definition}

As in the full update rule of~\eqref{eq:ofp},
we assume  $\Q$ encodes a tiebreaking
rule to ensure $\Q(z)$ returns a single element
from $\calW$. As in Remark~\ref{remark:tiebreak},
we make no assumptions on how such ties are broken.

\paragraph{Primal-dual dynamics.}
Using the definitions of the primal and dual variables
$\{w^t\}$ and $\{z^t\}$ and the choice map $\Q$,
the primal-dual dynamics of~\eqref{eq:ofp} 
can be rewritten as follows:
for all $t \ge 2$, let
% \footnote{
%   Recall from the initialization of~\eqref{eq:ofp}
%   that $x^0_1, x^0_2 \in \Delta_2$,
%   and thus $w^0$ is not necessarily in $\calW$.
%   However, $w^t \in \calW$ for all $t \ge 1$,
%   and thus update rule in
%   expression~\eqref{eq:subspace-ofp-1} holds
%   for all $t \ge 2$. 
% }
\begin{equation}
  \begin{cases}
    \;z^t
    \;=\; z^{t-1} + Sw^{t-1} \\
    \;w^t
    \;=\; \Q\big(z^t + S w^{t-1}\big)
  \end{cases}
  \quad\text{where}\quad
  \setlength\arraycolsep{3pt}
  S \;=\;
  \begin{pmatrix}
    b & a & a & b \\
    -c & -c & -a & -a
  \end{pmatrix} \;.
  \label{eq:subspace-ofp-1}
\end{equation}
Observe that the $i$'th column of $S \in \R^{2 \times 4}$
are the entries $(\Delta y^t_{11}, \Delta y^t_{12})$
when $w^t = e_i$ (cf., the definition of
$w^t$ from expression~\eqref{eq:wt}).
Moreover, expression~\eqref{eq:subspace-ofp-1}
implies $Sw^{t-1} = z^t - z^{t-1}$ for all
$t \ge 2$. Thus, we can further describe the
primal-dual dynamics of~\eqref{eq:subspace-ofp-1}
solely in terms of the sequence of dual variables $\{z^t\}$.
In particular, for all $t \ge 2$, we have
\begin{equation}
  \begin{cases}
    \;\widetilde z^{t+1}
    \;=\;
    z^t + (z^t - z^{t-1}) \\
    \;z^{t+1}
    \;=\;
    z^t + S\Q(\widetilde z^{t+1}) 
  \end{cases}
  \;.
  \tag{OFP Dual}
  \label{eq:ss-dual}
\end{equation}

\noindent\textbf{Energy function.}
% \paragraph{Energy function.}
Using the new notation and dual variables $\{z^t\}$, 
we define an energy function $\psi: \R^2 \to \R$
over the new dual space $\R^2$:
let $\psi((0,0)) = 0$, and for all other $z \in \R^2$, let
\begin{equation}
  \psi(z)
  \;=\;
  \begin{cases}
    - \rho_1 \cdot z_1 + z_2
    &\text{if $z \in \widehat P_2$,}
      \qquad\quad
      z_1 + z_2\quad\quad\;\;
    \text{if $z \in \widehat P_3$,} \\ 
    - \rho_1 \cdot z_1 - \rho_2 \cdot z_2\quad
    &\text{if $z \in \widehat P_1$,}
      \qquad\quad
    z_1 - \rho_2 \cdot z_2\quad
    \text{if $z \in \widehat P_4$}
  \end{cases} \;.
  \label{eq:energy-ss}
\end{equation}
% \todo{\{todo: more explanation of boundary conditions; figure;
%   relationship between $\psi(z)$ and $\|z\|_1$.\}} 
It is also more expressive to write
$\psi$ using the choice map $\Q$ from
Definition~\ref{def:Q-map}.
Recall for any $z \in \R^2$ that
$\Q(z) \in \calW = \{e_1, e_2, e_3, e_4\}$.
Then we have the following equivalent definition of $\psi$:
\begin{definition}
  \label{def:psi}
  For $z \in \R^2$, the function $\psi: \R^2 \to \R$
  is given by
  \begin{equation}
    \psi(z)
    \;=\;
    \big\langle
    z, M \Q(z)
    \big\rangle
    \qquad\text{where}\quad
    M
    \;=\;
    \begin{pmatrix}
      -\rho_1 & -\rho_1 & 1 & 1 \\
      -\rho_2 & 1 & 1 & -\rho_2 
    \end{pmatrix} \;.
    \label{eq:energy-ss-compact}
  \end{equation}
\end{definition}

% Thus we can rewrite $\psi$ as
% \begin{equation}
%   \psi(z)
%   \;=\;
%   \big\langle
%   z, M \Q(z)
%   \big\rangle
%   \qquad\text{where}\quad
%   M
%   \;=\;
%   \begin{pmatrix}
%     -\rho_1 & -\rho_1 & 1 & 1 \\
%     -\rho_2 & 1 & 1 & -\rho_2 
%   \end{pmatrix} \;.
%   \label{eq:energy-ss-compact}
% \end{equation}

Due to the relationship between
the coordinates of $y^t_1$ and $y^t_2$ from
Proposition~\ref{prop:subspace},
the following relationship between $\Psi$ and $\psi$
(and by Proposition~\ref{prop:energy-regret}, between
$\Psi$ and $\reg(T)$) is immediate:
\begin{proposition}
  \label{prop:energy-equiv}
  For all $t \ge 1$, $\Psi(y^t) = \psi(z^t)$.
  Moreover, $\reg(T) = \Psi(y^{T+1})= \psi(z^{T+1})$. 
\end{proposition}

\subsection{Bounded Energy of Subspace Dynamics}
\label{sec:ofp-regret:bounded-energy}

Leveraging Proposition~\ref{prop:energy-equiv},
it suffices to derive an upper bound
on $\psi(z^{T+1})$ in order to prove
Theorem~\ref{thm:ofp-energy-2x2}.
For this, the key step is to prove a set of
invariants on the dual iterates which establish that,
if the magnitude of $\psi$ ever
exceeds some constant threshold,
the subsequent one-step change in $\psi$
is non-increasing
(thus crystalizing the intuition
from expression~\eqref{eq:ofp-energy-step}). 

For this, we first define an absolute constant $B$ as follows:
\begin{restatable}{definition}{defB}
  \label{def:B}
  Fix $A$ satisfying Assumption~\ref{ass:2x2},
  and recall $\amax = \|A\|_\infty$.
  Then define $B > 0$ by
  \begin{equation*}
    B = \min\big\{
    b \in \R_+ :\;
    \text{for all $z \in \R^2$, }
      \|z\|_1 \le 6 \amax \implies \psi(z) \le b
    \big\}
    \;.
  \end{equation*}
\end{restatable}
In words, $B$ is the smallest constant
whose sublevel set $\psi(z) \le B$ contains
an $\ell_1$ ball of radius $6\amax$.
Moreover, the magnitude of $B$ can be bounded
from above as follows
(see Section~\ref{app:ofp-regret} for the proof):

\begin{restatable}{proposition}{propBupper}
  \label{prop:B-upperbound}
  Let $B$ be the constant from Definition~\ref{def:B}.
  Then $B \le 6\amax(1 + \rho_1 + \rho_2)$. 
\end{restatable}

\paragraph{Worst-case upper bound on energy.}
Observe that if $\psi(z^t) \le B$ for all times $t$,
then the statement of Theorem~\ref{thm:ofp-energy-2x2}
trivially holds.
On the other hand, if the energy $\psi$ crosses
this threshold under one step of the dynamics,
then we have the following constant upper bound
on $\psi(z^{t+1})$:

\begin{restatable}{lemma}{lemenergyupperbound}
  \label{lem:energy-upper-bound}
  Suppose $\psi(z^t) \le B$,
  and let $B' = 8\amax(1+\rho_1 + \rho_2)^2$.
  Then $\psi(z^{t+1}) \le B'$
\end{restatable}

\paragraph{Cycling invariants and controlled energy growth.}
The remaining step is to then control the
energy growth whenever $\psi(z^t) > B$.
For this, we prove the following key lemma:

\begin{restatable}{lemma}{leminvariantsenergy}
  \label{lem:invariantsenergy}
  Suppose $\psi(z^t) > B$ and $z^t \in \widehat P_i$
  for $i \in [4]$. Then the following hold:
  \begin{enumerate}[
    label={(\arabic*)},
    leftmargin=3em,
    topsep=0em,
    ]
  \item
    Either
    (i)  $\widetilde z^{t+1}, z^{t+1} \in \widehat P_i$ 
    or (ii) $\widetilde z^{t+1}, z^{t+1} \in P_{i+1}$
  \item
    $\Delta \psi(z^{t}) = \psi(z^{t+1}) - \psi(z^t) \;\le\; 0$.
  \end{enumerate}
\end{restatable}

Part (1) of the lemma establishes invariants
relating the true payoff vectors and predicted payoff vectors
whenever energy is above the threshold $B$.
Roughly speaking, when $\psi$ is larger than $B$,
the dual vectors cycle consecutively
through the regions $\widehat P_1, \dots, \widehat P_4$
(similarly to the iterates of standard Fictitious Play),
and this roughly implies that
$\Q(\widetilde z^{t+1}) = \Q(z^{t+1})$.

Importantly, this alignment between $z^{t+1}$
and $\widetilde z^{t+1}$ is the key
step needed to establish a non-increasing change in energy,
as stated in part (2).
To see why this is true, observe that
using the definition of $\psi$ from~\eqref{eq:energy-ss-compact},
we can compute the one-step change
$\Delta \psi(z^t) = \psi(z^{t+1}) - \psi(z^t)$
under~\eqref{eq:ss-dual} as
\begin{align}
  \Delta \psi(z^t)
  % \;=\;
  %   \psi(z^{t+1}) - \psi(z^t) 
  &\;=\;
    \big\langle
    z^{t+1}, M\Q(z^{t+1})
    \big\rangle
    -
    \big\langle
    z^t, M\Q(z^t) \big\rangle
    \label{eq:v0-delta-psi-0}\\
  &\;=\;
    \big\langle
    z^t + S\Q(\widetilde z^{t+1}),
    M\Q(z^{t+1})
    \big\rangle
    -
    \big\langle
    z^t, M\Q(z^t) \big\rangle \\
  &\;=\;
    \underbrace{
    \big\langle
    z^t, M\big(\Q(z^{t+1}) - \Q(z^t)\big)
    \big\rangle}_\text{(a)}
    +
    \underbrace{
    \big\langle
    \Q(\widetilde z^{t+1}), S^\top M \Q(z^{t+1})
    \big\rangle}_\text{(b)}
    \;.
    \label{eq:v0-delta-psi}
\end{align}
Here, \eqref{eq:v0-delta-psi}
essentially encodes the expression for $\Delta \Psi$
from~\eqref{eq:ofp-energy-step}. 
In particular, straightforward calculations show that
the matrix  $S^\top M$ is skew-symmetric,
and as Part (1) of the lemma roughly implies
$\Q(z^{t+1}) = \Q(\widetilde z^{t+1})$,
term (b) of~\eqref{eq:v0-delta-psi} vanishes.
Together with the column structure of $M$, the invariants
of part (1) imply that 
part (a) of~\eqref{eq:v0-delta-psi}
is non-positive, and thus overall $\Delta \psi(z^t) \le 0$.

The full proofs of the preceding lemmas
are developed in Section~\ref{app:ofp-regret}
and account more carefully for boundary conditions
and tiebreaking.
Figure~\ref{fig:ofp-subspace} of 
Section~\ref{app:ofp-regret:main-lemma}
also gives more visual intuition for the 
invariants and energy-growth behavior of Lemma~\ref{lem:invariantsenergy}.
Granting these lemmas as true for now,
we then give the proof of Theorem~\ref{thm:ofp-energy-2x2}:

\begin{proof}
  [\textnormal{\bf Proof of Theorem~\ref{thm:ofp-energy-2x2}}]
  Suppose for $t > 0$ that 
  $\psi(z^{t-1}) \le B$ and $\psi(z^{t}) > B$.
  By Lemma~\ref{lem:energy-upper-bound},
  we must have $\psi(z^t) \le 8\amax(1+\rho_1+\rho_2)^2$.
  Moreover, Lemmas~\ref{lem:invariants}
  and~\ref{lem:energy-cases} together imply that
  $\Delta \psi(z^t) \le 0$, and thus
  also $\psi(z^{t+1}) \le 8\amax(1+\rho_1+\rho_2)^2$.
  It follows inductively that
  $\psi(z^{T+1}) \le 8\amax(1+\rho_1+\rho_2)^2$.
  By definition, $\rho_1, \rho_2 \le (\amax/\agap)$,
  and thus we conclude
  \begin{equation*}
    \psi(z^{T+1})
    = \Psi(y^{T+1}) 
    \le 8\amax(1+2(\amax/\agap))^2 \;. \qedhere
\end{equation*} 
\end{proof}

%%% Local Variables:
%%% mode: latex
%%% TeX-master: "../ofp-neurips-main"
%%% End:

\section{Discussion and Conclusion}
\label{sec:conclusion}

In this work, we established for the first time
that the \textit{unregularized}
Optimistic Fictitious Play algorithm
can obtain \textit{constant} $O(1)$
regret in two-player zero-sum games.
Our proof technique leverages
a geometric viewpoint of Fictitious Play
algorithms, and we believe the techniques established
for the 2$\times$2 regime can be extended
to higher dimensions.

\paragraph{Additional
experimental results.}
To this end, in Table~\ref{table:comparison-01-lex}
we present additional experimental evidence
indicating that constant 
regret bounds for Optimistic FP
(similar to Theorem~\ref{thm:ofp-reg})
hold more generally in higher-dimensional settings. 
The table shows the empirical regret of
Optimistic FP and standard FP
(using lexicographical tiebreaking)
on three classes of zero-sum games,
in three higher dimensional settings.
For each setting, the algorithms were run
from 100 random initializations,
each for $T=10000$ iterations, and we report the
average regret over all initializations.
\begin{table}[h!]
  \centering
  \small
  \renewcommand{\arraystretch}{1.2}
  \setlength{\tabcolsep}{0.5em}
\begin{tabular}{c|| c c | c c | c c }
\textit{Dimension}:       & \multicolumn{2}{c | }{\textbf{15$\times$15}} & \multicolumn{2}{c| }{\textbf{25$\times$25}}   & \multicolumn{2}{c}{\textbf{50$\times$50}}   \\ 
\hline
\textit{Payoff Matrix} $\downarrow$     & \textbf{FP}           & \textbf{OFP}        & \textbf{FP}           & \textbf{OFP}         & \textbf{FP}          & \textbf{OFP}         \\
\hline 
\textbf{Identity}          & 155.1 $\pm$ 3.9 & 8.1 $\pm$ 1.6 & 161.3 $\pm$ 3.1 & 12.5 $\pm$ 1.7 & 167.2 $\pm$ 2.5 & 25.2 $\pm$ 2.1 \\
\textbf{RPS}               & 235.6 $\pm$ 7.6 & 2.9 $\pm$ 0.5 & 242.2 $\pm$ 6.3 & 2.9 $\pm$ 0.9  & 242.7 $\pm$ 5.9 & 2.5 $\pm$ 0.8  \\
\textbf{Random {[}0, 1{]}} & 116.2 $\pm$ 5.8 & 4.3 $\pm$ 0.8 & 118.6 $\pm$ 5.7 & 5.7 $\pm$ 0.9  & 177.0 $\pm$ 6.5 & 13.0 $\pm$ 1.5
\end{tabular}
\vspace*{0.5em}
\caption{\small
Empirical regret of FP and OFP using lexicographical tiebreaking.
Each entry reports an average and standard deviation
(over 100 random initializations) of total regret after $T=10000$ steps.}
\vspace*{-1em}
\label{table:comparison-01-lex}
\end{table}

The results indicate that, in each class of payoff matrix
and in each dimension, Opimistic FP has 
only constant regret compared to the regret of roughly
$\sqrt{T} \approx 100$ obtained by standard FP.
In Table~\ref{table:comparison-02-rand} of
Section~\ref{app:experiments}, 
we also report results using \textit{randomized tiebreaking}
for both algorithms and find similar conclusions,
thus highlighting the robustness of the constant regret
of OFP to tiebreaking rules. 
In Section~\ref{app:experiments}, we give more
details on the experimental setup
and additional plots similar to
Figure~\ref{fig:regret-comp-intro}.

\noindent\textbf{Limitations.}
% \paragraph{Open question.}
Formally proving whether Optimistic Fictitious Play
obtains constant regret in all zero-sum games remains 
an important open question.

\noindent\textbf{Broader impact.}
We acknowledge that there are many potential societal 
consequences of our theoretical results,
however none of which we feel must be specifically highlighted.

\paragraph{Acknowledgements.}
JL and RS were supported by the 
MOE Tier 2 Grant (MOE-T2EP20223-0018), 
the National Research Foundation, Singapore, 
under its QEP2.0 programme (NRF2021-QEP2-02-P05),
the National Research Foundation Singapore and
DSO National Laboratories under the 
AI Singapore Programme (Award Number: AISG2-RP-2020-016),
and SS was supported by the Villum Young Investigator Award
(Grant no. 72091).
The authors thank Anas Barakat and
Andre Wibisono for helpful discussions.

%%% Local Variables:
%%% mode: latex
%%% TeX-master: "../ofp-neurips-main"
%%% End:

\bibliographystyle{apalike}
\bibliography{references}

\newpage
\appendix
\setcounter{tocdepth}{2}
\renewcommand*\contentsname{Table of Contents}
\renewcommand{\baselinestretch}{0.5}\normalsize
\tableofcontents
\renewcommand{\baselinestretch}{1.0}\normalsize
\clearpage

% appendices
% Appendix for Prelims

\section{Details on Regret, Energy, and Fictitious Play}
\label{app:prelims}

\subsection{Zero-Sum Games and Convergence to Nash Equilibrium}

\propregretnash*

\begin{proof}
By definition of $\reg(T)$ from~\eqref{eq:regret}, 
we have that
\begin{align*}
    \frac{\reg(T)}{T}
    \;=\;
    \max_{x \in \Delta_m}\;
    \big\langle x, A \Big(\frac{1}{T}\sum\nolimits_{t=0}^T x^t_2\Big) \big\rangle
    -
    \min_{x \in \Delta_n}\;
    \big\langle x, A^\top \Big(\frac{1}{T}\sum\nolimits_{t=0}^T x^t_1\Big)\big\rangle 
    \;=\;
    \DG(\widetilde x^T_1, \widetilde x^T_2) \;,
\end{align*}
where we use the definitions of $\widetilde x^T_1$ and $\widetilde x^t_2$, 
and of the duality gap $\DG(\cdot, \cdot)$.
Thus if $\reg(T) = o(\sqrt{T})$, then
$\DG(\widetilde x^T_1, \widetilde x^T_2) = \frac{\reg(T)}{T} = o(1)$. 
\end{proof}

\subsection{Proof of Proposition~\ref{prop:energy-regret}}

In this section, we prove Proposition~\ref{prop:energy-regret},
which shows the equivalence between energy and regret.
Restated here:

\propenergyregret*

\begin{proof}
  For convenience, we let $d = m+n$, and we write
  $\calX = \Delta_m \times \Delta_n$.
  Then recall from Definition~\ref{def:Psi} that 
  for all $y \in \R^d$, the energy function $\Psi : \R^d \to \R$ 
  is given by
  \begin{equation}
    \Psi(y)
    \;=\;
    \max_{x = (x_1, x_2) \in \calX}\;
    \langle x, y \rangle
    \label{eq:energy}
  \end{equation}
  for $y = (y_1, y_2) \in \R^d$.
  Using the definitions of regret from
  from~\eqref{eq:regret} and of the dual variables
  from Definition~\ref{def:ofp-pd}, we have
  \begin{align*}
    \reg(T)
    &\;=\;
      \max_{x_1 \in \Delta_m}\;
      \Big\langle x_1, \sum_{t=1}^T Ax^t_2 \Big\rangle
      -
      \min_{x_1 \in \Delta_n}\;
      \Big\langle x_2, \sum_{t=1}^T A^\top x^t_1\Big\rangle \\
    &\;=\;
      \max_{x_1 \in \Delta_m}\;
      \big\langle x_1, y^{T+1}_1\big\rangle +
      \max_{x_2 \in \Delta_n}\;
      \big\langle x_2, y^{T+1}_2\big\rangle \\
    &\;=\;
      \max_{x = (x_1,x_2) \in \calX} \; \langle x, y^{T+1} \rangle \\
    &\;=\; \Psi\big(y^{T+1}\big) \;. \qedhere
  \end{align*}
\end{proof}

\subsection{Details on Fictitious Play Variants as Skew-Gradient Descent}
\label{app:prelims:energy:sgd}

In this section, we give more details on the 
geometric viewpoint of Optimistic FP and standard FP 
introduced in Section~\ref{sec:main-result}. 

\paragraph{Dual dynamics of fictitious play variants.}
First, we recall that for a convex function $H: \R^d \to \R$
that its subgradient set at $y \in \R^d$ is defined as
\begin{equation}
    \partial H(y) 
    \;=\;
    \big\{ g \in \R^d : \forall z \in \R^d, 
    H(z) \ge H(y) + \langle g, z -y \rangle \} \;.
    \label{eq:subgradient}
\end{equation}
Let $d = m+ n$. Then for the energy function $\Psi$ 
from Definition~\ref{def:Psi}, it follows that,
for any $y \in \R^d$, the subgradient set 
$\partial\Psi(y)$ is the set of maximizers
$\partial \Psi(y) =
\argmax_{x \in \Delta_m \times \Delta_n} \langle x, y \rangle$.
The next proposition (originally stated in Section~\ref{sec:main-result})
then follows by (1) using the definition of 
standard ($\alpha=0$) and Optimistic FP ($\alpha=1$) 
from~\eqref{eq:alpha-ofp}, 
and (2) by the definition of the dual payoff vectors. 

\propskewgradient*

\paragraph{One-step energy growth of FP variants.}
Using Proposition~\ref{prop:skew-gradient},
we can then derive the bounds on the one-step energy 
growth under FP and Optimistic FP,
as stated in expressions~\eqref{eq:fp-energy-step} and 
\eqref{eq:ofp-energy-step}.

For standard FP, using the convexity of $\Psi$ 
and Jensen's inequality (equivalently, the subgradient 
definition from~\eqref{eq:subgradient}), 
and letting $\partial \Psi(y)$ denote a fixed vector 
in the subgradient set of $\Psi$ at $y$, we have
for all $t \ge 1$
\begin{align*}
    \Delta \Psi(y^t) 
    \;=\;
    \Psi(y^{t+1}) - \Psi(y^t) 
    &\;\ge\;
    \big\langle
    \partial \Psi(y^t), 
    y^{t+1} - y^t
    \big\rangle \\
    &\;=\;
    \big\langle
    \partial \Psi(y^t), 
    J \partial \Psi(\widetilde y^{t+1}) 
    \big\rangle
    \;=\;
    \big\langle
    \partial \Psi(y^t), 
    J \partial \Psi(y^t) 
    \big\rangle
    \;=\;
    0 \;.
\end{align*}
Here, the first two equalities follow by the inductive update rule
for FP from Proposition~\ref{prop:skew-gradient}, 
and the final equality follows by skew-symmetry of $J=-J^\top$
(since $\langle y, Jy\rangle = 0$ for all $y \in \R^d$). 

For Optimistic FP, again using the subgradient definition 
of expression~\eqref{eq:subgradient}, we have for $t \ge 1$:
\begin{align*}
    \Delta \Psi(y^t) 
    \;=\;
    \Psi(y^{t+1}) - \Psi(y^t) 
    &\;\le\;
    \big\langle
    \partial \Psi(y^{t+1}), 
    y^{t+1} - y^t
    \big\rangle \\
    &\;=\;
    \big\langle
    \partial \Psi(y^{t+1}), 
    J \partial \Psi(\widetilde y^{t+1}) 
    \big\rangle \;,
\end{align*}
where the equality uses the update rule from 
Proposition~\ref{prop:skew-gradient} for Optimistic FP. 
Thus, by the skew-symmetry of $J$, 
and as explained in Section~\ref{sec:main-result}, 
the energy growth $\Delta \Psi(y^t)$ for Optimistic FP
is non-positive whenever 
\begin{equation*}
\partial \Psi(y^{t+1}) 
 \;=\; \partial \Psi(\widetilde y^{t+1})
  \;=\; \partial \Psi(2y^{t} - y^{t-1}) \;.
\end{equation*}

%%% Local Variables:
%%% mode: latex
%%% TeX-master: "../ofp-neurips-main"
%%% End:

% Assumptions on payoff matrix

\section{Assumptions on Payoff Matrix}
\label{app:wlog}

Recall from Section~\ref{sec:ofp-regret} that
in the proof of Theorem~\ref{thm:ofp-energy-2x2} for the 
2$\times$2 setting, we make the following assumption on the entries of the
payoff matrix:

\assA*

In this section, we give the
proofs of Proposition~\ref{prop:A-ass-wlog}
and Proposition~\ref{prop:subspace}.
Proposition~\ref{prop:A-ass-wlog} establishes
that the conditions of Assumption~\ref{ass:2x2} 
hold without loss of generality,
and Proposition~\ref{prop:subspace}
derives the resulting subspace invariance property
of the payoff vectors under $A$. 

\subsection{Proof of Proposition~\ref{prop:A-ass-wlog}}

We restate the proposition here: 

\propAwlog*

The proof of Proposition~\ref{prop:A-ass-wlog}
follows from the arguments in 
\citet[Appendix D]{bailey2019fast}.
For completeness, we re-derive the full proof here.

\begin{proof}
  Fix $A$, and let $(x^*_1, x^*_2) \in \Delta_2 \times \Delta_2$
  denote its unique, interior equilibrium.
  Then the coordinates of $x^*_1$ and $x^*_2$ are given by
  \begin{equation}
      x^*_1 =
      \bigg(
      \frac{d-c}{a+d-(b+c)},
      \frac{a-b}{a+d-(b+c)}
      \bigg)
      \quad
      % \;\text{and}\;-
      x^*_2 =
      \bigg(
      \frac{d-b}{a+d-(b+c)},
      \frac{a-c}{a+d-(b+c)}
      \bigg)\,.
    \label{eq:A-nash}
  \end{equation}
  
  Suppose that $A$ does not satisfy the conditions
  of Assumption~\ref{ass:2x2}.
  We will then construct $\widetilde A \in \R^{2\times 2}$ 
  that both satisfies the assumption, and such that the 
  two claims of the proposition statement hold.

  For this, suppose that the entries of $A$
  are shifted by the same additive constant $c$,
  and define the best responses $v$ and $v'$
  \begin{align}
    v &:= \argmax_{ e_i \in \{e_1, e_2\}}
        \big \langle
        e_i, Ax
        \big\rangle \\
    v' &:= \argmax_{ e_i \in \{e_1, e_2\}}
        \big \langle
        e_i, (A + c\1)x 
        \big\rangle
        =
        \argmax_{ e_i \in \{e_1, e_2\}}
        \big \langle
        e_i, Ax 
        \big\rangle
        + c \;,
  \end{align}
  where $x \in\Delta_2$ and $\1 \in \R^{2\times 2}$
  is the matrix of all ones.
  Thus for a fixed sequence of tiebreaking rules
  (e.g., the same adversarially-chosen tiebreak direction
  that is applied to determine
  $v$ is also applied to determine $v'$),
  it follows that the primal iterates of running
  \eqref{eq:ofp} on $A$ will be identical to
  those of running~\eqref{eq:ofp} on $(A+c\1)$
  (and note the same argument holds for
  any FTRL algorithm or variant, including
  standard Fictitious Play and Alternating Fictitious Play). 

  Now suppose $\det A \neq 0$. Let $\widetilde A$
  be the matrix
  \begin{equation*}
    \widetilde A
    = A - \Big(\frac{\det A}{a+d-(b+c)}\Big)\cdot \1\;.
  \end{equation*}
  By straightforward calculations, it follows
  that $\det \widetilde A = 0$.
  Moreover, by~\eqref{eq:A-nash}, $\widetilde A$
  has the same unique interior Nash equilibrium as $A$,
  and by the arguments above,
  the iterates of running~\eqref{eq:ofp} on $\widetilde A$
  are equivalent to those on $A$.
  Thus without loss of generality, we assume $\det A = 0$. 

  We now establish that we can
  assume $a > \max \{0, b, c\}$  without loss of generality.
  First, we show $a \neq 0$ holds:
  by the assumption that $\det A = ad - bc = 0$,
  if $a =0$, then $bc =0$.
  However, by~\eqref{eq:A-nash} and the
  assumption that $(x^*_1, x^*_2)$ is interior,
  we must have $a-c \neq 0 \implies c \neq 0$,
  which implies $b = 0$. 
  This violates the constraint from~\eqref{eq:A-nash}
  that $a-b = b \neq 0$, and thus without loss of
  generality, $a \neq 0$.
  To show without loss of generality that also $a > 0$,
  observe that the bilinear objective of the zero-sum game 
  is given by
  \begin{equation*}
    \max_{x_1 \in \Delta_2}
    \min_{x_2 \in \Delta_2}
    \langle x_1, Ax_2 \rangle
    =
    \max_{x_1 \in \Delta_2}
    \min_{x_2 \in \Delta_2}
    - \langle x_1, - Ax_2 \rangle
    =
    - \max_{x_2 \in \Delta_2}
    \min_{x_1 \in \Delta_2}
    \langle x_2, -A^\top x_1 \rangle
    \;.
  \end{equation*}
  Thus, by switching the maximization
  or minimization role between the players
  (via scaling the matrix by -1),
  we may assume that $a > 0$.
  Finally, to show $a > \max\{b, c\}$ holds
  without loss of generality, observe
  from~\eqref{eq:A-nash} that
  if $a+d-(b+c) > 0$, then
  the interior Nash condition in~\eqref{eq:A-nash} implies
  $a > c$ and $a > b$.
  If instead $a+d - (b+c) < 0$, then
  $0 < a < \min\{b, c\}$, 
  and we can then rewrite the bilinear objective of
  the zero-sum game using a new payoff matrix with
  relabeled strategies (i.e., permuting the columns of $A$), as
  \begin{equation}
    \max_{x_1 \in \Delta_2}
    \min_{x_2 \in \Delta_2}
    \langle x_1, Ax_2 \rangle
    =
    \max_{x_1 \in \Delta_2}
    \min_{x_2 \in \Delta_2}
    \langle x_1, A' x_2 \rangle
    \quad
    \text{where}\;
    A' =
    \begin{pmatrix}
      b & a \\
      d & c 
    \end{pmatrix}\;.
  \end{equation}
  Under $A'$, we have $b+c- (a+d) > 0$,
  which from~\eqref{eq:A-nash} and the reasoning
  above implies $b > \max\{a, d\} > 0$.
  As a consequence, by possibly permuting
  the columns of $A$ and relabeling
  the strategies of Player 1, we can assume in either
  case that $a > \max\{b, c\}$.
  Together, we conclude that the assumption
  $a > \max\{0, b, c\}$ holds without loss of generality.

  Similarly, it follows that we may also assume
  $d > \max\{0, b, c\}$ without loss of generality.
  Specifically, using the relabeling argument above,
  we may assume $a + d - (b+c) > 0$. 
  Then under the unique interior Nash 
  and $\det A = 0$ assumptions, it follows
  from~\eqref{eq:A-nash} (using similar arguments
  as for $a \neq 0$) that $d \neq 0$ and $d > \max \{b, c\}$.
  Moreover, as $a > 0$, if also $d < 0$, then
  this implies $c, b < 0$, meaning
  $\det A = ab - cd < 0$, contradicting the 
  assumption that $\det A = 0$.
  Thus also $d > 0$, and we conclude
  that the assumption $d > \max \{0, b, c\}$
  holds without loss of generality. 
\end{proof}

\subsection{Proof of Proposition~\ref{prop:subspace}}

We restate the proposition here for convenience:

\propsubspace*

\begin{proof}
  For player 1, let $v_1 = (d-c, a-b)$. Then observe that 
  \begin{equation*}
    A^\top v_1 
    \;=\;
    \begin{pmatrix}
        a & c \\
        b & d 
    \end{pmatrix}
    \begin{pmatrix}
        d-c \\
        a-b
    \end{pmatrix}
    \;=\;
    \begin{pmatrix}
        ad - ac + ac - bc \\
        bd - bc + ad - bd 
    \end{pmatrix}
    \;=\;
    \begin{pmatrix}
        0 \\
        0 
    \end{pmatrix} \;,
  \end{equation*}
  where the final equality follows from the 
  assumption that $\det A = ab - cd = 0$. 
  Then for any $x \in \Delta_2$, we have
  \begin{equation*}
    0 \;=\; \langle x, A^\top v_1 \rangle
    \;=\;
    \langle v_1, A x \rangle. 
  \end{equation*}
  As $y^t_1 = \sum_{k=1}^{t-1} A x^k_2$,
  this implies
  \begin{equation*}
    \langle v_1, y^t_1 \rangle
    \;=\;
    \sum_{k=1}^{t-1}
    \langle v_1, A x^k_2 \rangle
    \;=\;
    0 \;.
  \end{equation*}
  Thus for all $t$, we have
  $
  \langle v_1, y^t_1 \rangle
   =
   (d-c)\cdot y^t_{11} + (a-b)\cdot y^t_{12} = 0$ .
  Rearranging, and recalling that $\rho_1 := (d-c)/(a-b) > 0$
  (where the inequality follows by Assumption~\ref{ass:2x2}),
  we find $y^t_{12} = - \rho_1 \cdot y^t_{11}$.
  
  For the second player, let $v_2 = (d-b, a-c)$.
  Using a similar argument and calculation, we have 
  $A v_2 = 0 \in \R^2$ and thus 
  \begin{equation*}
    \langle v_2, y^t_2 \rangle
    \;=\;
    \sum_{k=1}^{t-1}
    \langle v_2, -A^\top x^k_1 \rangle
    \;=\;
    \sum_{k=1}^{t-1}
    \langle x^k_1, -A v_2 \rangle
    \;=\;
    0 \;.
  \end{equation*}
  For all $t$, it then follows that
  $\langle v_2, y^t_2 \rangle
  =
  (d -b)\cdot y^t_{21} + (a-c)\cdot y^t_{22} = 0$,
  meaning $y^t_{22} = -\rho_2  \cdot y^t_{21}$. 
\end{proof}

%%% Local Variables:
%%% mode: latex
%%% TeX-master: "../ofp-neurips-main"
%%% End:

% Appendix for ofp 2x2

\section{Proofs for Optimistic Fictitious Play Regret Upper Bound}
\label{app:ofp-regret}

In this section, we develop the omitted proofs from 
Section~\ref{sec:ofp-regret} that are needed to establish
the main technical result of Theorem~\ref{thm:ofp-energy-2x2}
(showing Optimistic FP has bounded energy in 2$\times$2 games). 

\subsection{Properties of the Energy Threshold $B$}
\label{app:ofp-regret:B}

In this section, we prove several properties
related to the threshold $B$ that is used 
in the proof of Theorem~\ref{thm:ofp-energy-2x2}:
Recall that $B$ is defined as follows:

\defB*

First, we prove the following upper bound on
the magnitude of $B$ with respect to the
constants $\amax, \rho_1$, and $\rho_2$:

\propBupper*

\begin{proof}
  By definition of $B$, the level set
  $\calL = \{z \in R^2 : \psi(z) = B\}$ must intersect
  the boundary of the ball $\calB = \{z \in \R^2 : \|z\|_1 \le B\}$
  on at least one of the boundaries $P_{i\sim(i+1)}$
  in the set $\widetilde \calP$ from Definition~\ref{def:P-sets}.
  Let $\calI = \widetilde \calP \cap \calB \cap \calL$ be
  the intersection of these three sets.
  Using the definition of $\psi$ from \eqref{eq:energy-ss}
  it follows that for $z \in \calI$
  \begin{equation*}
    B \;=\; \psi(z) \;=\;
    \begin{cases}
      \; \rho_1 \cdot |z_1|
      = \rho_1 \cdot 6 \amax
      &\text{if $z \in P_{1\sim 2}$} \\
      \; |z_2| = 6 \amax 
      &\text{if $z \in P_{2\sim 3}$} \\
      \; |z_1| = 6 \amax
      &\text{if $z \in P_{3\sim 4}$} \\
      \; \rho_2 \cdot |z_2|
      = \rho_2 \cdot 6\amax 
      &\text{if $z \in P_{4\sim 1}$} 
    \end{cases} \;,
  \end{equation*}
  where in each case the equality comes from the
  fact that if $z \in \calI$ then $\|z\|_1 = 6\amax$.
  It follows that
  \begin{equation*}
    B
    \;\le\;
    6\amax \cdot 
    \max \big\{
    1, \rho_1, \rho_2
    \big\}
    \;\le\;
    6\amax (1 + \rho_1 + \rho_2) \;,
  \end{equation*}
  where the final inequality comes from the fact
  that $\rho_1, \rho_2 > 0$. 
\end{proof}

Next, we establish the following invariant:

\begin{proposition}
  \label{prop:B-2step}
  Let $B$ be the constant from Definition~\ref{def:P-sets}.
  Suppose $\psi(z) > B$ and suppose $z \in \widehat P_i \cup P_{(i-1)~i}$
  for some $i \in [4]$.
  Assume either $\widetilde z = z + S_j$
  or $\widetilde z = z + S_j + S_k$
  for $j, k \in [4]$ and $S$ as in~\eqref{eq:ss-dual}.
  Then
  \begin{equation}
    \widetilde z \notin
    \widehat P_{i+2}
    \cup
    P_{(i+1)\sim(i+2)}
    \;.
  \end{equation}
\end{proposition}

\begin{proof}
  We prove the claim for the case that $\widetilde z = z + S_j + S_k$,
  which by the same argument implies the result
  when $\widetilde z = z + S_j$. 
  Without loss of generality, assume $i =1$.
  Under the assumptions of the proposition,
  we will show that if $z \in \widehat P_1 \cup P_{4 \sim 1}$,
  then $\widetilde z \notin \widehat P_3 \cup P_{2\sim 3}$.
  For this, observe first by definition of $B$ that
  if $\psi(z) > B$ then $\|z\|_1 > 6\amax$.
  By definition of the sets $\widehat P_1$
  and $\widehat P_3 \cup P_{2 \sim 3}$, this implies that
  \begin{equation}
    \min_{z' \in \widehat P_3 \cup P_{2 \sim 3}}
    \|z - z'\|_2
    \;\ge\;
    \|z\|_2
    \;\ge\;
    \frac{1}{\sqrt{2}} \|z\|_1
    \;>\;
    \frac{6\amax}{\sqrt{2}}
    \ge 4 \amax \;.
    \label{eq:min-z-dist}
  \end{equation}
  On the other hand, by construction of $\widetilde z$,
  and using the fact that $\|S\|_2 \le 2\amax$, we have
  \begin{equation}
    \|z - \widetilde z\|_2
    \;\le\;
    \|S_j\|_2 + \|S_k\|_2
    \;\le\;
    2 (2\amax) = 4\amax \;.
    \label{eq:z-ztilde-dist}
  \end{equation}
  Then combining expressions~\eqref{eq:min-z-dist}
  and~\eqref{eq:z-ztilde-dist}, we find
  \begin{equation*}
    \|z - \widetilde z\|_2
    \;<\;
    \min_{z' \in \widehat P_3 \cup P_{2 \sim 3}}
    \|z - z'\|_2 \;,
  \end{equation*}
  and thus $\widetilde z \notin \widehat P_3 \cup P_{2 \sim 3}$.
\end{proof}

\subsection{Energy Upper Bound:
  Proof of Lemma~\ref{lem:energy-upper-bound}}
\label{app:ofp-regret:energy} 

This section gives the proof of
Lemma~\ref{lem:energy-upper-bound}, which derives
an upper bound on the energy $\psi(z^{t+1})$
when $\psi(z^t) \le B$. Restated here:

\lemenergyupperbound*

\begin{proof}
  Using the definition of $\psi$ from~\eqref{eq:energy-ss},
  observe that
  \begin{equation}
    \psi(z^{t+1})
    \;\le\;
    \max \big\{
    \max(\rho_1, \rho_2)\cdot \|z^{t+1}\|_1,
    \|z^{t+1}\|_1
    \big\}
    \;\le\;
    (1 + \rho_1 + \rho_2) \cdot \|z^{t+1}\|_1 \;.
    \label{eq:psi-ub-01}
  \end{equation}
  Now recall by definition of the constant $B$
  that $\psi(z^t) \le B \implies \|z^t\|_1 \le B$. 
  Then as $\|z^{t+1}\|_1 = \|z^t + S_j\|_1$
  for some $j \in [4]$, we have that
  \begin{align}
    \|z^{t+1}\|_1
    \;\le\;
    \|z^t\|_1 + 2\amax
    &\;\le\;
      B + 2\amax \nonumber \\
    &\;\le\;
    6\amax(1+\rho_1 + \rho_2)
    + 2\amax \nonumber \\
    &\;\le\;
    8\amax(1+\rho_1 + \rho_2) \;.
    \label{eq:psi-ub-02}
  \end{align}
  Here, the penultimate inequality follows
  from the upper bound on $B$ from Proposition~\ref{prop:B-upperbound},
  and the final inequality follows
  from the positivity of $\rho_1, \rho_2$.
  
  Combining expressions~\eqref{eq:psi-ub-01}
  and~\eqref{eq:psi-ub-02}, we conclude that
  \begin{equation*}
    \psi(z^{t+1})
    \;\le\;
    8\amax (1+\rho_1 + \rho_2)^2 \;. \qedhere
  \end{equation*}
\end{proof}

\subsection{Cycling Invariants and Non-Increasing
  Energy Growth: Proof of Lemma~\ref{lem:invariantsenergy}}
\label{app:ofp-regret:main-lemma}

In this section, we develop the proof of
Lemma~\ref{lem:invariantsenergy}, restated here:

\leminvariantsenergy*

We give the proof of Lemma~\ref{lem:invariantsenergy}
in two parts:
first in Lemma~\ref{lem:invariants}
(Section~\ref{app:ofp-regret:invariants}), 
we prove the invariants from Part (1).
Then, in Lemma~\ref{lem:energy-cases}
(Section~\ref{app:ofp-regret:energy-decrease}),
we prove the non-positive energy growth bounds
from Part (2).

\begin{figure}[h!]
\centering
\includegraphics[width=0.95\textwidth]{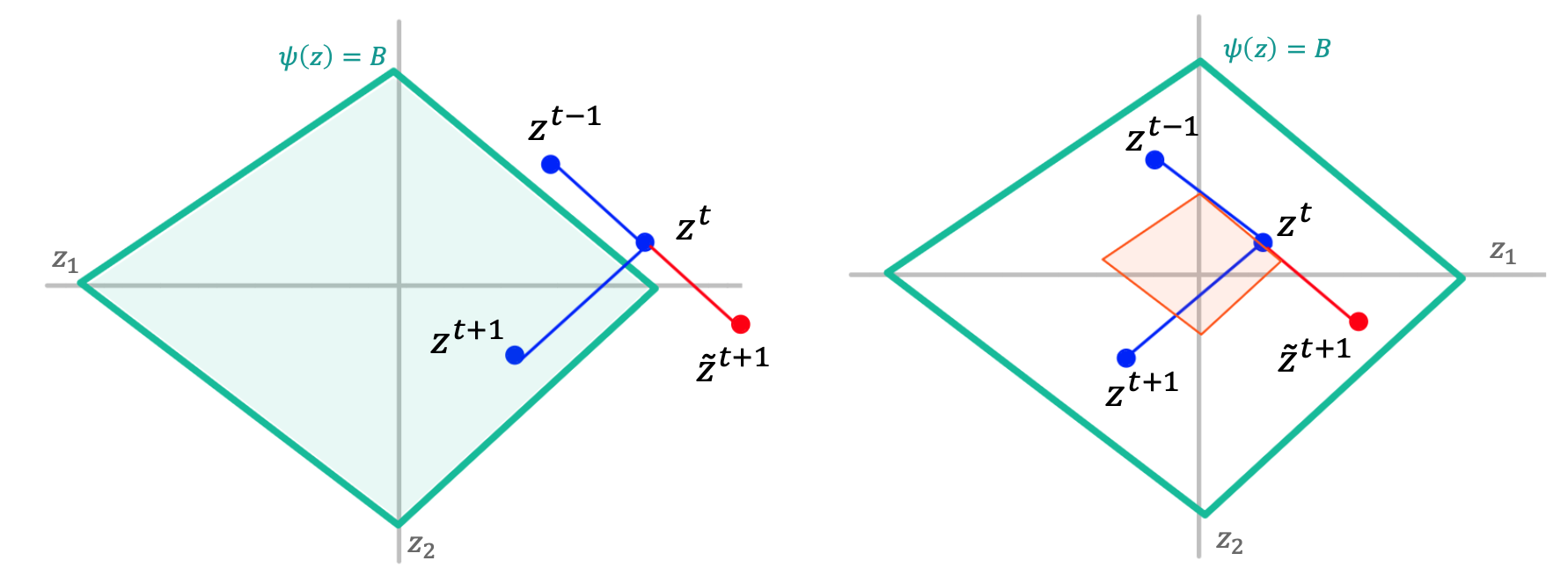}
\caption{Visual intuition for the claims and proof of 
Lemma~\ref{lem:invariantsenergy}.}
\label{fig:ofp-subspace}
\end{figure}

In Figure~\ref{fig:ofp-subspace}, we also give visual intuition
for the two claims of Lemma~\ref{lem:invariantsenergy}.
In the figure, each subfigure shows the dual space $\R^2$, 
and the green region denotes the sublevel set $\psi(z) \le B$.
The left subfigure illustrates that for $\psi(z^t) > B$, 
the vectors  $z^{t+1}$ and $\widetilde z^{t+1}$ will lie in the same 
region of $\calP$, and thus $\Delta \psi(z^t) \le 0$
(the latter point is captured by the fact that
$z^{t+1}$ lies within the sublevel set $\psi(z) \le B$). 
In contrast, as illustrated in the right subfigure, 
when $\psi(z^t) \le B$, then in general the invariants
of Part (1) of the lemma may not hold, and $\Delta \psi(z^t)$
can be strictly positive. 
  
\subsubsection{Cycling Invariants:
  Part 1 of Lemma~\ref{lem:invariantsenergy}}
\label{app:ofp-regret:invariants}

In this section, we state and prove 
the following lemma, which establishes
certain ``cycling invariants'' on the relationship
between the predicted cost vector $\widetilde z^{t+1}$
and true cost vector $z^{t+1}$ that hold
when the energy $\psi(z^t)$ is sufficiently large.

\begin{restatable}{lemma}{leminvariants}
  \label{lem:invariants}
  Suppose $\psi(z^t) > B$ and $z^t \in \widehat P_i$
  for $i \in [4]$. Then either
  \begin{enumerate}[
    label={(\roman*)},
    topsep=0em,
    leftmargin=3em,
    ]
  \item
    $\widetilde z^{t+1} \in \widehat P_i$ and
    $z^{t+1} \in \widehat P_i$, or 
  \item
    $\widetilde z^{t+1} \in P_{i+1}$ and
    $z^{t+1} \in P_{i+1}$.
  \end{enumerate}
\end{restatable}

\paragraph{Proof of Lemma~\ref{lem:invariants}.}

The proof of Lemma~\ref{lem:invariants}
proceeds using three separate propositions,
which we state and prove as follows:

First, we establish regions in which
$\widetilde z^{t+1}$ and $z^{t+1}$ cannot lie
when the energy $\psi$ is sufficiently large. 

\begin{proposition}
  \label{prop:invariant-regions}
  Suppose $\psi(z^t) > B$ and $z^t \in \widehat P_i$
  for some $i \in [4]$.
  Then
  \begin{equation*}
    \widetilde z^{t+1}, z^{t+1}
    \;\notin\; \widehat P_{i+2} \cup P_{(i+1)\sim (i+2)}
    \cup \widehat P_{i-1}\;.
\end{equation*}
\end{proposition}

\begin{proof}
  We first show that 
  $\widetilde z^{t+1}, z^{t+1} \notin
  \widehat P_{i+2} \cup P_{(i+1)\sim(i+2)}$.
  For this, recall that both $z^{t+1} = z^t + S_j$
  and $\widetilde z^{t+1} = z^t + S_k$
  for some $j, k \in [4]$. As $\psi(z^t) > B$,
  then applying Proposition~\ref{prop:B-2step}
  implies that both
  $z^{t+1}, \widetilde z^{t+1}
  \notin \widehat P_{i+2} \cup P_{(i+1)\sim(i+2)}$.
  
  Next, we establish that $z^{t+1} \notin \widehat P_{i-1}$.
  For this, as $z^t \in \widehat P_i$,
  then either (a) $\widetilde z^{t+1} \in \widehat P_{i-1}$
  or (b) $\widetilde z^{t+1} \in \widehat P_{i+2} \cup P_{(i+1)\sim(i+2)}$.
  We have already established case (b) cannot hold.  
  Similarly, as $\psi(z^t) > B$ implies $\|z^t\|_1 > 3\amax$,
  then case (a) implies that either $\psi(z^t) \le B$ or
  $z^{t+1} \notin \widehat P_{i-1}$, which is a contradiction.
  So $z^{t+1} \notin \widehat P_{i-1}$.
  
  Similarly, to have $\widetilde z^{t+1} \in \widehat P_{i-1}$,
  then by definition of $S$ we must have either
  (a) $z^t - z^{t-1} = S_{i+2}$ or
  (b) $z^t - z^{t-1} = S_{i+1}$.
  Case (a) implies
  $\widetilde z^t \in \widehat P_{i+2} \cup P_{(i+1)\sim(i+2)}$,
  which cannot hold due to Proposition~\ref{prop:B-2step}.
  Case (b) implies $\widetilde z^t \in
  \widehat P_{i+1} \cup P_{i \sim(i+1)}$, which again
  due to the definition of $B$ contradicts that either
  $\psi(z^t) > B$ or that $\widetilde z^{t+1} \in \widehat P_{i-1}$.
  Thus we conclude
  $\widetilde z^{t+1} \notin \widehat P_{i-1}$. 
\end{proof}

Proposition~\ref{prop:invariant-regions} establishes
that if $\psi(z^t) > B$, then we must have
$\widetilde z^{t+1}, z^{t+1} \in \widehat P_i \cup P_{i+1}$.
Thus to conclude the proof of Lemma~\ref{lem:invariants}, 
it suffices to establish that
$\widetilde z^{t+1} \in \widehat P_i \implies z^{t+1} \in \widehat P_i$
and that
$\widetilde z^{t+1} \in P_{i+1} \implies z^{t+1} \in P_{i+1}$.
We prove these claims in the following two propositions:

\begin{proposition}
  \label{prop:invariant-a}
  Suppose $\psi(z^t) > B$ and $z^t \in \widehat P_i$
  for some $i \in [4]$. Then:
  \begin{equation*}
    \widetilde z^{t+1} \in \widehat P_i
    \implies
    z^{t+1} \in \widehat P_i \;.
  \end{equation*}
\end{proposition}

\begin{proof}
  We distinguish the cases when
  $\widetilde z^{t+1} \in P_i$ and $\widetilde z^{t+1} \in P_{i\sim(i+1)}$.
  In the first case, if $\widetilde z^{t+1} \in P_{i}$,
  then by definition $z^{t+1} = z^t + S_i$.
  If $\psi(z^{t+1}) = \psi(z^t)$, then
  $z^{t+1} = \widetilde z^{t+1} \in P_i$.
  If instead $\psi(z^{t+1}) \neq  \psi(z^t)$,
  then $\widetilde z^{t+1} = z^t - z^{t-1} = S_j$
  for some $j \neq i \in [4]$.
  However, by definition of the constant $B$,
  we must have $z^{t+1} \notin P_{i+1}$, as otherwise
  the assumption $\psi(z^t) > B$ would be violated.
  Then Proposition~\ref{prop:invariant-regions} implies 
  $z^{t+1} \in \widehat P_i$.
  Thus if $\widetilde z^{t+1} \in P_i$,
  then $z^{t+1} \in \widehat P_i$.

  For the second case, suppose
  $\widetilde z^{t+1} \in P_{i \sim (i+1)}$.
  Then we must have $\widetilde z^{t+1} - z^t \in \{S_i, S_{i+1}\}$.
  Moreover, recall that $\Q(\widetilde z^{t+1}) \in \{e_i, e_{i+1}\}$.
  In either case, using the structure of adjacent
  columns $i$ and $i+1$ of $S$, it follows that
  $z^{t+1} \in P_{i \sim (i+1)} \subset \widehat P_i$. 
  Thus if $\widetilde z^{t+1} \in P_{i\sim(i+1)}$,
  then also $z^{t+1} \in \widehat P_i$, which concludes the proof.  
\end{proof}

\begin{proposition}
  \label{prop:invariant-b}
  Suppose $\psi(z^t) > B$ and $z^t \in \widehat P_i$
  for some $i \in [4]$. Then:
  \begin{equation*}
    \widetilde z^{t+1} \in P_{i+1}
    \implies
    z^{t+1} \in P_{i+1} \;.
  \end{equation*}
\end{proposition}

\begin{proof}
  Suppose $z^t \in P_i$.
  If $\widetilde z^{t+1} \in P_{i+1}$,
  then $\Q(z^{t+1}) = e_{i+1}$, and also
  $\widetilde z^{t+1}-z^t \in \{S_{i}, S_{i+1}\}$.
  Using the structure of adjacent columns $i$ and $i+1$
  of $S$, it then follows that
  $z^{t+1} = z^t + S_{i+1} \in P_{i+1}$.

  Similarly, if instead $z^t \in P_{i\sim(i+1)}$,
  then $z^{t+1} = z^t + S_{i+1} \in P_{i+1}$
  by definition of $S$.
  Thus in either case, if $\widetilde z^{t+1} \in P_{i+1}$,
  then also $z^{t+1} \in P_{i+1}$. 
\end{proof}

\subsubsection{Non-Increasing Energy Growth:
Part (2) of Lemma~\ref{lem:invariantsenergy}}
\label{app:ofp-regret:energy-decrease}
  % Proof of Lemma~\ref{lem:energy-upper-bound}}
 % \label{app:ofp-regret:energy} 
 
In this section we state and prove the following lemma,
which gives non-positive bounds on the energy growth
under the two cases in Part (1) of Lemma~\ref{lem:invariantsenergy}:

\begin{restatable}{lemma}{lemenergycases}
  \label{lem:energy-cases}
  Fix $i \in [4]$, and suppose $z^t \in \widehat P_i$.
  Suppose that either
  (i)  $\widetilde z^{t+1} \in \widehat P_i$
  and $z^{t+1} \in \widehat P_i$
  or (ii) $\widetilde z^{t+1} \in P_{i+1}$
  and $z^{t+1} \in P_{i+1}$.
  Then $\Delta \psi(z^t) \le 0$. 
\end{restatable}

\begin{proof}
  To start, we rederive the one-step change
  in energy growth under~\eqref{eq:ss-dual}.

  \noindent
  \textbf{One-step change in energy:}
  Using~\eqref{eq:ss-dual} and Definition~\ref{def:psi}, 
  we have for any $t\ge 1$:
  \begin{align}
    \Delta \psi(z^t)
    &\;=\;
      \psi(z^{t+1}) - \psi(z^t) \label{eq:delta-psi-0}\\
    &\;=\;
      \big\langle
      z^{t+1}, M\Q(z^{t+1})
      \big\rangle
      -
      \big\langle
      z^t, M\Q(z^t) \big\rangle \\
    &\;=\;
      \big\langle
      z^t + S\Q(\widetilde z^{t+1}),
      M\Q(z^{t+1})
      \big\rangle
      -
      \big\langle
      z^t, M\Q(z^t) \big\rangle \\
    &\;=\;
      \underbrace{
      \big\langle
      z^t, M\big(\Q(z^{t+1}) - \Q(z^t)\big)
      \big\rangle}_\text{(a)}
      +
      \underbrace{
      \big\langle
      \Q(\widetilde z^{t+1}), S^\top M \Q(z^{t+1})
      \big\rangle}_\text{(b)}
      \;.
      \label{eq:delta-psi}
  \end{align}
  By the definitions of $S$ and $M$
  from expressions~\eqref{eq:ss-dual} and~\eqref{eq:energy-ss-compact}, 
  respectively, recalling that $\rho_1 = (d-c)/(a-b)$ and 
  $\rho_2 = (d-b)/(a-c)$, 
  and using the fact from Assumption~\ref{ass:2x2}
  that $\det A = ab - cd = 0$, we can compute
  \begin{equation}
    S^\top M
    \;=\;
    \begin{pmatrix}
      b & -c \\
      a & -c \\
      a & -a \\
      b & -a
    \end{pmatrix}
    \begin{pmatrix}
      -\rho_1 & -\rho_1 & 1 & 1 \\
      -\rho_2 & 1 & 1 & -\rho_2 \\      
    \end{pmatrix}
    \;=\;
    \begin{pmatrix}
      0 & d-c & b-c & b-d \\
      c-d & 0 & a-c & a-d \\
      b-c & c-a & 0 & a-b \\
      d-b & d-a & b-a & 0
    \end{pmatrix} \;.
    \label{eq:STM}
  \end{equation}
  Thus, expression~\eqref{eq:STM} shows
  $S^\top M$ is skew-symmetric.

  \noindent 
  \textbf{Proof for Case (i)}:
  To prove the claim for case (i) of the lemma,
  we start with the case that $z^t \in P_i$
  and also $\widetilde z^{t+1}, z^{t+1} \in P_i$.
  Then by definition of $\Q$, we have
  $\Q(z^t) = \Q(\widetilde z^{t+1}) = \Q(\widetilde z^t) = e_i$. 
  By skew-symmetry of $S^\top M$, observe in part (b)
  of expression~\eqref{eq:delta-psi} that
  \begin{equation*}
    \big\langle
    \Q(\widetilde z^{t+1}), S^\top M \Q(z^{t+1})
    \big\rangle
    \;=\;
    \big\langle
    \Q(z^{t+1}), S^\top M \Q(z^{t+1})
    \big\rangle
    \;=\;
    0 \;.
  \end{equation*}
  Moreover, in part (a) of expression~\eqref{eq:delta-psi},
  we also have
  \begin{equation*}
    \big\langle
    z^t, M\big(\Q(z^{t+1}) - \Q(z^t) \big)
    \big\rangle
    \;=\;
    \big\langle
    z^t, M\big(\Q(z^{t}) - \Q(z^t) \big)
    \big\rangle
    \;=\;
    0 \;,
  \end{equation*}
  and thus $\Delta \psi(z^t) = 0$.
  In the case that $z^t \in P_i$
  and $\widetilde z^{t+1} \in P_{i\sim(i+1)}$,
  then observe from the structure of $S$ that
  we must also have $z^{t+1} \in P_{i\sim(i+1)}$.
  Then by definition of $\psi$, for any
  $z \in P_{i \sim (i+1)}$, we have 
  $\psi(z) = \langle z, M e_i \rangle$. Thus we
  can rewrite expression~\eqref{eq:delta-psi-0} as
  \begin{align}
    \Delta \psi(z^t)
    &\;=\;
    \langle z^t + S\Q(\widetilde z^{t+1}), M e_i \rangle
    -
    \langle z^t, M\Q(z^t) \rangle \\
    &\;=\;
      \langle z^t, M(e_i - \Q(z^t))\rangle
      +
      \langle \Q(\widetilde z^{t+1}), S^\top M e_i \rangle \;.
      \label{eq:delta-psi-boundary} 
  \end{align}
  As $z^t \in P_i \implies \Q(z^t) = e_i$, the first term
  above vanishes.
  Moreover, as $\widetilde z^{t+1} \in P_{i\sim(i+1)}$,
  we have $\Q(\widetilde z^{t+1}) \in \{e_i, e_{i+1}\}$.
  By skew-symmetry of $S^\top M$, if
  $\Q(\widetilde z^{t+1}) = e_i$, then
  the second term of \eqref{eq:delta-psi-boundary} also vanishes.
  On the other hand, if $\Q(\widetilde z^{t+1}) = e_{i+1}$,
  then the second term is negative, which
  follows from the fact that, under Assumption~\ref{ass:2x2},
  each entry $(S^\top M)_{i+1, i} < 0$.
  In either case, we find $\Delta \psi(z^t) \le 0$.

  Finally, observe that if $z^t \in P_{i \sim(i+1)}$,
  then by definition of~\eqref{eq:ss-dual}, we cannot have
  both $\widetilde z^{t+1}, z^{t+1} \in \widehat P_i$.
  Thus the conditions of case (i) do not apply, 
  which concludes the proof of the lemma under case (i). 

  \ss

  \noindent 
  \textbf{Proof for Case (ii)}:
  To prove the claim for case (ii), suppose first that $z^t \in P_i$
  and thus $\Q(z^t) = e_i$.
  By the assumptions of claim (ii), we also have
  $\Q(\widetilde z^{t+1}) = \Q(z^{t+1}) = e_{i+1}$.
  Thus it again follows by skew-symmetry of $S^\top M$
  that for part (b) of expression~\eqref{eq:delta-psi}
  \begin{equation}
    \big\langle
    \Q(\widetilde z^{t+1}), S^\top M \Q(z^{t+1})
    \big\rangle
    \;=\;
    0 \;.
    \label{eq:energy-ii-1}
  \end{equation}
  For part (a) of~\eqref{eq:delta-psi},
  by case analyis on the columns of $M$,
  it follows that when $\Q(z^t) = e_i$
  and $\Q(z^{t+1}) = e_{i+1}$, then
  \begin{equation}
    \big\langle
    z^t, M\big(\Q(z^{t+1}) - \Q(z^t)\big)
    \big\rangle
    \;=\;
    \begin{cases}
      (1+\rho_2) \cdot z^t_2
      &\text{if $\Q(z^t) = e_1$} \\
      (1+\rho_1) \cdot z^t_1
      &\text{if $\Q(z^t) = e_2$} \\
      -(1+\rho_2) \cdot z^t_2
      &\text{if $\Q(z^t) = e_3$} \\
      -(1+\rho_1) \cdot z^t_1
      &\text{if $\Q(z^t) = e_4$} 
    \end{cases} \;.
    \label{eq:part-a-cases}
  \end{equation}
  Given the definition of $\Q$,
  it follows that $\Q(z^t) = e_1 \implies z^t_2 \le 0$,
  that $\Q(z^t) = e_2 \implies z^t_1 \le 0$,
  that $\Q(z^t) = e_3 \implies z^t_2 \ge 0$,
  and that $\Q(z^t) = e_4 \implies z^t_1 \ge 0$.
  Together with the fact that $\rho_1, \rho_2 > 0$ by definition,
  in each case of expression~\eqref{eq:part-a-cases}, we find
  $\big\langle
  z^t, M\big(\Q(z^{t+1}) - \Q(z^t)\big)
  \big\rangle \le 0$.
  Together with \eqref{eq:energy-ii-1},
  this means $\Delta \psi(z^t) \le 0$. 

  In the case that $z^t \in P_{i \sim(i+1)}$,
  then either $\Q(z^t) = e_i$ or $\Q(z^t) = e_{i+1}$.
  If the latter holds, given that
  also $\Q(z^{t+1}) = e_{i+1}$ by assumption,
  then part (a) of~\eqref{eq:delta-psi} is trivially 0.
  If the former holds, we recover the cases of
  expression~\eqref{eq:part-a-cases},
  and thus part (a) of~\eqref{eq:delta-psi} is
  non-positive.
  In either case, part (b) of~\eqref{eq:delta-psi}
  remains 0 as in expression~\eqref{eq:energy-ii-1},
  and thus we conclude that $\Delta \psi(z^t) \le 0$.
  This proves the lemma under case (ii).
\end{proof}

%%% Local Variables:
%%% mode: latex
%%% TeX-master: "../ofp-neurips-main"
%%% End:

% AFP lower bound appendix

\section{Proofs for Alternating Fictitious Play Regret Lower Bound}
\label{app:afp-lb}

In this section, we develop the proof of
Theorem~\ref{thm:afp-regret-lb}, which gives
a \textit{lower bound} of $\Omega(\sqrt{T})$
on the regret of \textit{Alternating}
Fictitious Play. Restated here:

\thmafpreglb*

The organization of this section is as follows:
in Section~\ref{app:afp-lb:setup} we recall the
setup of alternating play in zero-sum games,
as well as on the notion of alternating regret.
In Section~\ref{app:afp-lb:afp}, we formally
define the Alternating Fictitious Play algorithm.
Finally, in Section~\ref{app:afp-lb:proof}, we
give the proof of Theorem~\ref{thm:afp-regret-lb}. 

\subsection{Details on Alternating Play and Alternating Regret}
\label{app:afp-lb:setup}

\paragraph{Alternating play.}
We consider the model of alternating online learning in
two-player zero-sum games as in~\cite{bailey2020finite,
  wibisono2022alternating, katona2024symplectic}.
Defined formally:

\begin{definition}[Alternating Play]
  \label{def:alt-play}
  Fix a payoff matrix $A \in \R^{m\times n}$.
  Over $T$ rounds, Players 1 and 2 alternate
  updating their strategies $x^t_1 \in \Delta_m$
  and $x^t_2 \in \Delta_{n}$ as follows:
  \begin{itemize}[
    leftmargin=2em
    ]
  \item
    \textbf{(Initialization)}~      
    Assume without loss of generality  $T$ is even. 
    At time $t=1$, Player 1 chooses an initial $x^1_1 \in \Delta_m$,
    and Player 2 observes $-A^\top x^1_1$. 
  \item
    \textbf{(Even rounds -- Player 2 updates)}~
    When $t=2k$ (for $k \ge 1$):
    \begin{align*}
      &\text{Player 1 sets } x^t_1 = x^{t-1}_1 \in \Delta_m
        \qquad\quad
        \text{Player 2 updates } x^t_2 \in \Delta_n.\\
      &\text{Player 1 observes } Ax^t_2
        \qquad\qquad\qquad\;
        \text{Player 2 observes } -A^\top x^{t-1}_1.
    \end{align*}
  \item
    \textbf{(Odd rounds -- Player 1 updates)}~
    When $t=2k+1$ (for $k \ge 1$):
    \begin{align*}
      &\text{Player 1 updates } x^t_1 \in \Delta_m
        \qquad\qquad\;\;
        \text{Player 2 sets } x^t_2 = x^{t-1}_2 \in \Delta_n.\\
      &\text{Player 1 observes } Ax^{t-1}_2
        \qquad\qquad\quad\;
        \text{Player 2 observes } -A^\top x^{t}_1.
    \end{align*}
  \end{itemize}
\end{definition}

\paragraph{Alternating regret.}
Under alternating play, we now measure
the performance of each player by its
\textit{alternating regret}
\citep{wibisono2022alternating, cevher2023alternation,
  hait2025alternating}.
For this, first observe under alternating play that each
player's \textit{cumulative payoff} can be written as:
\begin{equation}
  \begin{aligned}
    &\text{Player 1 cumulative payoff:}\quad
      \sum_{k=1}^{T/2}
      \big\langle
      x^{2k-1}_1, A(x^{2k}_2 + x^{2k-2}_2) 
      \big\rangle \;.
    \\
    &\text{Player 2 cumulative payoff:}\quad
      \sum_{k=1}^{T/2}
      \big\langle
      x^{2k}_2,
      - A^\top(x^{2k+1}_1 + x^{2k-1}_1)
      \big\rangle \;.
  \end{aligned}
  \label{eq:alternating-payoffs}
\end{equation}
Here and throughout, we assume for notational convenience that
$x^0_2 = 0 \in \R^n$ and $x^{T+1}_1 = 0 \in \R^m$.
Then alternating regret is defined as follows:

\begin{definition}[Alternating Regret]
  \label{def:alt-regret}
  Let $T$ be even. Define
  $\regalt_1(T)$ and $\regalt_2(T)$ as
  \begin{equation*}
    \begin{aligned}
      \regalt_1(T)
      &\;=\;
        \max_{x \in \Delta_m}\;
        \sum_{k=1}^{T/2}
        \big\langle
        x - x^{2k-1}_1, A(x^{2k}_2 + x^{2k-2}_2) 
        \big\rangle \\
      \regalt_2(T)
      &\;=\;
        \min_{x \in \Delta_n}
        \sum_{k=1}^{T/2}
        \big\langle
        x^{2k}_2 - x,
        A^\top(x^{2k+1}_1 + x^{2k-1}_1)
        \big\rangle \;.
    \end{aligned} 
  \end{equation*}
  Then define $\regalt(T) = \regalt_1(T) + \regalt_2(T)$. 
\end{definition}

Similar to standard (simultaneous) play,
sublinear regret bounds for $\regalt(T)$
correspond to convergence of the time-average iterates
under alternating play to a Nash equilibrium of $A$.
For this, define the time-average iterates
$\widetilde x^T_1 \in \Delta_m$
and $\widetilde x^T_2 \in \Delta_n$ by
\begin{equation}
  \widetilde x^T_1
  \;=\;
  \frac{1}{T} \Big(\sum_{k=1}^{T/2}
  x^{2k-1}_1 + x^{2k+1}_1\Big)
  \quad\text{and}\quad
  \widetilde x^T_2
  \;=\;
  \frac{1}{T} \Big(\sum_{k=1}^{T/2}
  x^{2k-2}_2 + x^{2k+2}_2\Big) \;.
  \label{eq:alt-time-avg-iterates}
\end{equation}

Then we have the following proposition
(analogous to Proposition~\ref{prop:regret-nash}
for simultaneous play):

\begin{proposition}
  \label{prop:alt-regret-nash}
  Fix $A \in \R^{m \times n}$.
  Let $\widetilde x^T_1 \in \Delta_m$
  and $\widetilde x^T_2 \in \Delta_n$ denote
  the time-average iterates
  under the alternating play of
  Definition~\ref{def:alt-regret}, as
  in expression~\eqref{eq:alt-time-avg-iterates}.
  Suppose $\regalt(T) \le \alpha = o(T)$.
  Then $(\widetilde x^T_1, \widetilde x^T_2)$
  converges in duality gap to an NE
  of $A$ at a rate of $\alpha/T = o(1)$. 
\end{proposition}

\begin{proof}
  By definition of the player-wise cumulative
  costs from~\eqref{eq:alternating-payoffs}
  (and recalling that we set $x^0_2 = 0 \in \R^n$
  and $x^{T+1}_1 = 0 \in \R^m$ for notational
  convenience), observe that
  \begin{equation*}
    \sum_{k=1}^{T/2}
    \big\langle
    x^{2k-1}_1, A(x^{2k}_2 + x^{2k-2}_2) 
    \big\rangle
    \;+\;
    \sum_{k=1}^{T/2}
    \big\langle
    x^{2k}_2,
    - A^\top(x^{2k+1}_1 + x^{2k-1}_1)
    \big\rangle
    \;=\;
    0 \;.
  \end{equation*}
  It follows from the Definition~\ref{def:alt-regret} that
  \begin{align*}
    \regalt(T)
    &\;=\;
      \regalt_1(T) + \regalt_2(T) \\
    &\;=\;
      \max_{x \in \Delta_m}\;
      \sum_{k=1}^{T/2}
      \big\langle
      x, A(x^{2k}_2 + x^{2k-2}_2) 
      \big\rangle
      \;-\;
      \min_{x \in \Delta_n}
      \sum_{k=1}^{T/2}
      \big\langle
       x,
      A^\top(x^{2k+1}_1 + x^{2k-1}_1)
      \big\rangle \\
    &\;=\;
      \max_{x \in \Delta_m}\;
      \big\langle
      x, A(T\cdot \widetilde x^T_2)
      \big\rangle
      -
      \min_{x \in \Delta_n}\;
      \big\langle
      x, A^\top(T\cdot \widetilde x^T_1)
      \big\rangle
      \;\le\;
      \alpha \;,
  \end{align*}
  where in the final line we use
  the definition of $\widetilde x^T_1$
  and $\widetilde x^T_2$ from~\eqref{eq:alt-time-avg-iterates}
  and the assumption that $\reg(T) \le \alpha$.
  Then dividing by $T$ gives
  \begin{equation*}
     \DG(\widetilde x^T_1, \widetilde x^T_2) 
     \;=\;
     \max_{x \in \Delta_m}\;
      \big\langle
      x, A \widetilde x^T_2
      \big\rangle
      -
      \min_{x \in \Delta_n}\;
      \big\langle
      x, A^\top  \widetilde x^T_1
      \big\rangle
      \;\le\;
      \frac{\alpha}{T} \;,
    \end{equation*}
    which yields the statement of the proposition. 
\end{proof}

\subsection{Details on Alternating Fictitious Play}
\label{app:afp-lb:afp}

Under the alternating play setup of
Definition~\ref{def:alt-play}, we now
specify the Alternating Fictitious Play algorithm.
For any even $t \ge 2$, the primal iterates of Players 1 and 2 
at times $t+1$ and $t+2$ update according to
\begin{equation*}
    \begin{aligned}  
    x^{t+1}_1
    &\;:=\;
      \argmax_{x \in \{e_i\}_m}
      \big\langle
      x, \sum_{k=1}^{t/2} A (x^{2k}_2 + x^{2k-2}_2)
      \big\rangle
      \quad\text{and}\quad
      x^{t+1}_2 = x^{t}_2   \\
    x^{t+2}_2
    &\;:=\;
      \argmax_{x \in \{e_i\}_n}
      \big\langle
      x, \sum_{k=1}^{t/2} - A^\top (x^{2k+1}_1 + x^{2k-1}_2)
      \big\rangle 
      \quad\text{and}\quad
      x^{t+2}_1 = x^{t+1}_1 \;.
  \end{aligned}
\end{equation*}
In other words, as in standard Fictitious Play
(c.f.,~\eqref{eq:alpha-ofp} for $\alpha=0$),
in Alternating Fictitious Play 
each player (in an alternating fashion),
selects the best-response to the
cumulative observed payoff vectors over
all prior rounds.

\paragraph{Primal-Dual update for Alternating FP.}
Similar to the analysis for Optimistic FP,
define the dual payoff vectors
$y^t_1 = \sum_{k=1}^{t-1} Ax^k_2 \in \R^m$
and $y^t_2 = \sum_{k=1}^{t-1} -A^\top x^k_1 \in \R^n$.
Then the iterates of Alternating FP can be
equivalently expressed as follows:
\begin{restatable}{definition}{defafppd}
  \label{def:afp-pd}
  Assume the alternating play setting of
  Definition~\ref{def:alt-play}. 
  Let $y^2_1 = 0 \in \R^m$,
  and let $y^2_2 = -A^\top x^1_1 \in \R^n$. 
  Then for $t \ge 2$,
  the dual (i.e., $(y^t_1, y^t_2)$)
  and primal (i.e., $(x^t_1, x^t_2)$) iterates
  of Alternating FP are given by
  \begin{equation}
  \begin{aligned}
    \text{($t$ even)}
    \qquad
    &\begin{cases}
      x^t_1 = x^{t-1}_1 \\
      x^t_2 = \argmax_{x \in \{e_i\}_n}\;
      \langle x, y^t_2 \rangle
    \end{cases}
      \quad\text{and}\quad
      \begin{cases}
        y^{t+1}_1 = y^t_1 + Ax^t_2 \\
        y^{t+1}_2 = y^t_2 - A^\top x^{t-1}_1 \;.
      \end{cases} \\
    \text{($t$ odd)}
    \qquad
    &\begin{cases}
      x^t_1 = \argmax_{x \in \{e_i\}_m}\;
      \langle x, y^t_1 \rangle \\
      x^t_2 = x^{t-1}_2
    \end{cases}
      \quad\text{and}\quad
      \begin{cases}
        y^{t+1}_1 = y^t_1 + Ax^{t-1}_2 \\
        y^{t+1}_2 = y^t_2 - A^\top x^{t}_1  \;.
      \end{cases} 
  \end{aligned}
    \label{eq:afp} 
    \tag{AFP}
    \end{equation}
\end{restatable}

Moreover, recall the energy function
$\Psi$ from Definition~\ref{def:Psi}
and $\regalt(T)$ from Definition~\ref{def:alt-regret}.
Then, analogously to Proposition~\ref{prop:energy-regret},
following equivalence between
energy and alternating regret holds:
\begin{proposition}
  \label{prop:energy-regret-alt}
  Let $\{x^t\}$ and $\{y^t\}$
  be iterates of \eqref{eq:afp}.
  Then $\regalt(T) = \Psi(y^{T+1})$.
\end{proposition}

\subsection{Proof of Theorem~\ref{thm:afp-regret-lb}:
Regret Lower Bound on Matching Pennies}
\label{app:afp-lb:proof}

We now prove the lower bound on the regret
of~\eqref{eq:afp} on Matching Pennies. 
For this, recall that the Matching Pennies
payoff matrix is given by
\begin{equation}
  A \;=\;
  \begin{pmatrix}
    1 & -1 \\
    -1 & 1
  \end{pmatrix} \;.
  \label{eq:MP}
  \tag{Matching Pennies}
\end{equation}

\paragraph{Subspace Dynamics of AFP for Matching Pennies.}
It is straightforward to check that
\eqref{eq:MP} satisfies the conditions of
Assumption~\ref{ass:2x2}.
Moreover, this also implies
Proposition~\ref{prop:subspace} holds
for the dual iterates of \eqref{eq:afp},
in particular for $\rho_1 = \rho_2 = 1$.

Thus, to prove the theorem, we reuse the
components of the \textit{subspace dynamics}
introduced in Section~\ref{sec:ofp-regret}.
Specfically, we reuse the notation
of the primal and dual iterates $\{w^t\}$
and $\{z^t\}$, as well as the
choice map $\Q$ from Definition~\ref{def:Q-map},
and the energy $\psi$ from Definition~\ref{def:psi}. 

Under \eqref{eq:MP},
it is then straightforward to check that
the matrix $S$ from~\eqref{eq:subspace-ofp-1}
and the energy $\psi$ from Definition~\ref{def:psi} are given by:
\begin{equation*}
  S = \begin{pmatrix}
    -1 & 1 & 1 & -1 \\
    1 & 1 & -1 & - 1
  \end{pmatrix}
  \quad
  \text{and for all $z \in \R^2$:}\;
  \psi(z) = \|z\|_1 \;.
\end{equation*}
Similarly to Proposition~\ref{prop:energy-equiv},
and using the definition of $\psi$ and
the iterates $\{z^t\}$, 
we also have the following relationship between $\Psi$
and $\psi$:
\begin{proposition}
  \label{prop:energy-equiv-alt}
  Let $\{y^t\}$ be the iterates of
  \eqref{eq:afp} on~\eqref{eq:MP},
  and let $\{z^t\}$ be the corresponding
  subspace iterates.
  Then $\Psi(y^{T+1}) = \psi(z^{t+1})$. 
\end{proposition}

Moreover, under the primal-dual definition
of~\eqref{eq:subspace-ofp-1}
it follows inductively
(and using the definition of $\{w^t\}$,  $\{z^t\}$, and $\Q$)
that for all $t \ge 3$:
\begin{equation}
  w^t
  \;=\;
  \begin{cases}
    \Q((z^{t-1}_1, z^t_2))
    &\text{for $t$ even} \\
    \Q((z^{t}_1, z^{t-1}_2))
    &\text{for $t$ odd}
  \end{cases}\;\;.
  \label{eq:afp-wt-iterates}
\end{equation}
Then for $t \ge 3$ that the dual iterates $\{z^t\}$
can be further rewritten as
\begin{equation}
  z^{t+1}
  \;=\;
  z^t + S\Q(\widetilde z^{t+1})\quad
  \text{where}\;
  \widetilde z^{t+1} \;=\;
  \begin{cases}
    (z^{t-1}_1, z^t_2) &\text{for $t$ even} \\
    (z^{t}_1, z^{t-1}_2) &\text{for $t$ odd} 
  \end{cases} \;.
  \label{eq:afp-ss-dual}
  \tag{AFP Dual}
\end{equation}

Thus, similar to~\eqref{eq:afp-ss-dual},
the subspace iterates of Alterating Fictitious Play
can be expressed with respect to a
predicted payoff vector $\widetilde z^{t+1}$.
Now, due to the alternating play setting,
the position of this predicted vector depends on
the parity of $t$.

\paragraph{Overall proof strategy.}
Given the equivalence between
$\regalt(T)$ and $\Psi(y^{T+1})$
from Proposition~\ref{prop:energy-regret-alt},
and on the equivalence between
$\Psi(y^{T+1})$ and $\psi(z^{T+1})$
from \ref{prop:energy-equiv-alt},
to prove Theorem~\ref{thm:afp-regret-lb},
it suffices to establish the following lower
bound on the energy $\psi(z^{T+1})$:

\begin{lemma}
  \label{lem:afp-lb}
  Assume the setting of Theorem~\ref{thm:afp-regret-lb},
  and let $\{z^t\}$ be the dual iterates of
  \eqref{eq:afp-ss-dual}. 
  Then $\psi(z^{T+1}) \ge \Omega(\sqrt{T})$. 
\end{lemma}

To prove Lemma~\ref{lem:afp-lb}, we introduce
a \textit{phase structure} (in similar spirit to the
analysis of~\cite{lazarsfeld2025fp}),
where each phase tracks a subsequence of consecutive
time steps where the iterates $\{w^t\}$ are at the
same primal vertex. Formally, we define:

\begin{definition}
  \label{def:afp-phases}
  Let $\{w^t\}$ be the primal iterates
  from~\eqref{eq:afp-wt-iterates}, and fix $t_0 = 2$.
  For $k \ge 1$,
  let $t_k := \min\{t > t_{k-1} : w^t \neq w^{t_{k-1}}\}$. 
  Then define Phase $k$  as the subsequence of iterates
  from times $t=t_k, t_k+1 \dots, t_{k+1} - 1$,
  and let $\tau_k = t_{k+1} - t_k$ be the
  length of the phase.
  Let $K \ge 0$ denote the total number of phases
  in $T$ rounds such that $T = \sum_{k=0}^T \tau_k$. 
\end{definition}

Using the phase setup of Definition~\ref{def:afp-phases},
the core technical component of proving Lemma~\ref{lem:afp-lb}
is to establish the following proposition:
\begin{proposition}
  \label{prop:afp-lb-core}
  Assume the setting of Theorem~\ref{thm:afp-regret-lb}.
  Then for each Phase $k = 1, \dots, K$, the following hold:
  \begin{enumerate}[
    label={(\roman*)},
    topsep=0em,
     % leftmargin=3em,          
    ]
  \item
    $\psi(z^{t_k}) \le \psi(z^{t_{k-1}}) + 2$
  \item
    $\tau_k = \Theta(\psi(z^{t_k}))$.
  \end{enumerate}
  Moreover, for at least $K/2$ phases $k$, it holds
  that (iii) $\psi(z^{t_k}) \ge \psi(z^{t_{k-1}}) + 1$. 
\end{proposition}

The proof of Proposition~\ref{prop:afp-lb-core}
is developed in Section~\ref{app:afp-lb:proof-prop}.
Granting the claims of the proposition as true for now,
we give the proof of Lemma~\ref{lem:afp-lb}
(and thus also of Theorem~\ref{thm:afp-regret-lb}):

\begin{proof}[Proof (of Lemma~\ref{lem:afp-lb})]
  By claim (iii) of Proposition~\ref{prop:afp-lb-core},
  the energy $\psi$ is strictly increasing
  in at least $K/2$ phases, and thus
  \begin{equation}
    \psi(z^{T+1})
    \;\ge\;
    \frac{K}{2}  \;.
    \label{eq:prop-1}
  \end{equation}
  To prove the statement of the lemma,
  it then suffices to derive a lower bound on $K$. 
  For this, recall by Definition~\ref{def:afp-phases}
  that $T = \sum_{k=1}^K \tau_k$.
  Moreover, combining claims (i) and (ii) of 
  Proposition~\ref{prop:afp-lb-core}, we find
  for all $k$ that 
  $\tau_k = \Theta(\psi(z^{t_k}))
  \le \Theta(\psi(z^{t_{k-1}}) + 2)
  \le \Theta(k)$.
  Combining these pieces, we have
  \begin{align}
    T \;=\;
    \sum_{k=1}^K \tau_k
    \;\le\;
    \sum_{k=1}^K \Theta(k)
    \;\le\;
    \Theta(K^2)  \;.
  \end{align}
  Thus $K^2 \ge \Omega(T) \implies K \ge \Omega(\sqrt{T})$.
  Substituting into~\eqref{eq:prop-1},
  we conclude $\psi(z^{T+1}) \ge \Omega(\sqrt{T})$. 
\end{proof}

\subsubsection{Proof of Proposition~\ref{prop:afp-lb-core}}
\label{app:afp-lb:proof-prop}

We now prove the claims of Proposition~\ref{prop:afp-lb-core}.
For this, we start by establishing the following
invariant between the dual iterates
$z^{t-1}, z^t, z^{t+1}$ and the predicted vector
$\widetilde z^{t+1}$.

\paragraph{Analysis of initial phases.}
We begin by computing the dual iterates 
during the first two phases,
which helps to both give intuition for the energy
growth behavior of Alternating FP,
as well as to streamline the remainder of the proof.
For this, recall that initially
$x^1_2 = (p, 1-p) \in \Delta_2$ for irrational
$p \in (3/4, 1)$, and that $y^2_1 = 0 \in \R^2$.

It follows by definition of~\eqref{eq:afp}
at time $t=2$ that
$y^2_1 = y^2_1 = 0 \in \R^2$ and 
$y^2_2 = -A^\top x^1_1 = (-(2p-1), (2p-1))$,
and that $x^2_1 = x^1_1 \in \Delta_2$ and
$x^2_2 = (0, 1) \in \Delta_2$.
Then, at $t=3$, we further have
$y^3_1 = y^2_1 + Ax^2_2 = (-1,1)$
and $y^3_2 = y^2_2 - A^\top x^1_1 = 2 \cdot y^2_2$. 

Then for $t \ge 3$, switching to the equivalent,
lower-dimensional iterates $\{w^t\}$ and $\{z^t\}$,
we can further directly compute (by definition
of~\eqref{eq:afp-ss-dual}):
\begin{align*}
  &(t=3)\qquad
  \begin{cases}
    z^3
    =
    (-1, -2(2p-1)) \in P_1 \\
    \widetilde z^4
    =
    (z^3_1, z^2_2)
    =
    (-1, -(2p-1)) \in P_1
  \end{cases}
    &w^3 = \Q(\widetilde z^4) = e_1  \\
  &(t=4)\qquad
  \begin{cases}
    z^4 = z^3 + (-1, 1) = (-2, -4p+3)\in P_1 \\
    \widetilde z^5 = (z^3_1, z^4_2) = (-1, -4p+3) \in P_1
  \end{cases}
  &w^4 = \Q(\widetilde z^5) = e_1 \\
  &(t=5)\qquad
    \begin{cases}
      z^5 = z^4 + (-1, 1) = (-3, -4p + 4) \in P_2 \\
      \widetilde z^6 = (z^5_1, z^4_2) = (-3, -4p+3) \in P_1
    \end{cases}
  &w^5 = \Q(\widetilde z^6) = e_1 \\
  &(t=6)\qquad
    \begin{cases}
      z^6 = z^5 + (-1, 1) = (-4, -4p + 5) \in P_2 \\
      \widetilde z^7 = (z^5_1, z^6_2) = (-3, -4p+5) \in P_2
    \end{cases}
    &w^6 = \Q(\widetilde z^7) = e_2.
\end{align*}
Observe by Definition~\ref{def:afp-phases}
and the calculations above that 
Phase 1 begins at step $t_1 = 3$,
and Phase 2 begins at phase $t_2 = 6$.
Moreover, $\Delta \psi(z^5) = \Delta \psi(z^6) = 1$,
meaning $\psi(z^{t_2}) - \psi(z^{t_1}) = 2 > 0$. 

This strictly increasing energy growth between
phases stems from the geometry of the predicted
payoff vectors: in this instance,
under Alternating Fictitious Play,
when $z^t, z^{t-1}\in P_1$ and $z^t$ is near
the boundary $P_{1\sim 2}$, the predicted vector
$\widetilde z^{t+1}$ always remains in $P_1$
and fails to ``predict'' the next region $P_2$.
This results in strictly increasing energy
growth when $z^{t+1} \in P_2$.
This positive energy growth behavior near the
boundary regions is the key difference between 
Alternating and Optimistic Fictitious Play
(c.f., the invariants and energy growth claims 
of Lemma~\ref{lem:invariantsenergy}). 

\paragraph{Cycling invariants.}
By continuing to compute the dual iterates $\{z^t\}$,
we arrive at the following invariants, which
establish a certain cycling behavior through
the regions of $\widehat \calP$.
Specifically, it follows inductively that
$z^{t-1}$ and $z^t$ must fall under one of the
following cases (which subsequently determines
$\widetilde z^{t+1}$, $z^{t+1}$,
and the energy growth $\Delta \psi(z^t)$):
\begin{itemize}[
  leftmargin=2em,
  rightmargin=1em,
  ]
\item
  \textbf{Case 1:}
  $z^{t-1}, z^t \in P_i$,
  and $z^{t} - z^{t-1} = S_i$.

  Then $\widetilde z^{t+1} \in P_i$,
  and either $z^{t+1} \in P_{i+1}$
  with $\Delta \psi(z^t) = 1$,
  or $z^t \in P_i$ with $\Delta \psi(z^t) = 0$. 

\item
  \textbf{Case 2:}
  $z^{t-1} \in P_i$, $z^t \in P_{i+1}$,
  and $z^t = z^{t-1} + S_i$. 

  Then $z^{t+1} \in P_{i+1}$,
  and either $\widetilde z^{t+1} \in P_i$
  and $\Delta \psi(z^t) = 1$,
  or $\widetilde z^{t+1} \in P_{i+1}$
  and $\Delta \psi(z^t) = 0$. 

\item
  \textbf{Case 3:}
  $z^{t-1} \in P_i$, $z^t \in P_{i\sim(i+1)}$
  and $z^t = z^{t-1} + S_i$.

  If $\widetilde z^{t+1} \in P_i$,
  then $z^{t+1} \in P_{i+1}$ 
  and $\Delta \psi(z^t) = 1$.
  If $\widetilde z^{t+1} \in P_{i \sim(i+1)}$,
  then also $z^{t+1} \in P_{i+1}$,
  with $\Delta \psi(z^t) \in \{0, 1\}$
  depending on the tiebreaking of $\Q$. 
  
\item
  \textbf{Case 4:}
  $z^{t-1} \in P_{i\sim(i+1)}$ and $z^t \in P_{i+1}$.
  If $\widetilde z^{t+1} \in P_{i\sim(i+1)}$,
  then $z^{t+1} \in P_{i+1}$, and
  $\Delta \psi(z^t) \in \{0, 1\}$
  depending on the tiebreaking of $\Q$.
  If $\widetilde z^{t+1} \in P_{i+1}$,
  then $z^{t+1} \in P_{i+1}$ and $\Delta \psi(z^t) = 0$. 
\end{itemize}
Note that the cases above account
for (a) the variability of $\widetilde z^{t+1}$
depending on the parity of $t$, and
(b) any variability in $z^{t+1}$
depending on the tiebreaking decision encoded
in $\Q$. In summary, we deduce from the four
cases above the following consequences:
\begin{enumerate}[
  leftmargin=2em,
  ]
\item
  Between phases, energy strictly increases
  in at most 2 iterations. By definition of
  the energy function $\psi$ under Matching Pennies,
  each one-step increase has magnitude 1,
  and thus $\psi(z^{t_k}) - \psi(z^{t_{k-1}}) \le 2$,
  which proves claim (i) of the proposition.
\item
  Again using the definition of $\psi$ under
  Matching Pennies, we have $\psi(z^t) = \|z^t\|_1$.
  The cases above then imply 
  that each $\tau_k = \|z^t\|_1 + c_k$
  (for some aboslute constant $c_k$), and it follows
  that $\tau_k = \Theta(\psi(z^t))$,
  which proves claim (ii) of the proposition.
\item
  Finally, using the definition of $S$ under
  Matching Pennies, along with the fact
  that initially $z^2_1 = 0$, it holds that
  each $z^t_1$ is integral. Thus
  between regions $P_2$ and $P_3$, and
  between $P_4$ and $P_1$, one dual iterate
  will always lie on the boundary $P_{2\sim 3}$
  or $P_{3\sim 4}$, respectively. In these cases,
  depending on the tiebreaking rule of $\Q$,
  the change in energy may be zero when crossing
  between regions of $\calP$. 
  On the other hand, due to the initialization
  $x^1_1 = (p, 1-p)$ for irrational $p \in (3/4, 1)$,
  it follows for $t \ge 2$ that all $z^t_2$ are irrational.
  Thus no tiebreaking occurs when the dual
  iterates transition between regions $P_1$ and $P_2$
  and between $P_3$ and $P_4$.
  Thus under transitions between these phases
  (which by symmetry amount for at least $\Omega(K/2)$
  total phases), 
  we have by the cases above that energy is
  strictly increasing by at least 1.
  This proves claim (iii) of the proposition.
  \hfill
  $\qed$
\end{enumerate}

%%% Local Variables:
%%% mode: latex
%%% TeX-master: "../ofp-neurips-main"
%%% End:

% appendix for experiments

\section{Additional Experimental Results}
\label{app:experiments}

In this section, we provide more details
on the experimental evaluations from 
Figure~\ref{fig:regret-comp-intro} 
and Section~\ref{sec:conclusion}, and 
we also present additional experimental results. 
The goal of these experiments is to give further
empirical evidence that the 
constant regret guarantee of Theorem~\ref{thm:ofp-energy-2x2} 
for two-strategy games also holds 
in higher dimensions. 

\subsection{Details on Experimental Setup}
\label{app:experiments:setup}

% compute resource disclaimer
First, we note that all code used to run
experiments can be found in the supplementary material.
In this paper, all experiments were run locally on 
a single personal computer.

\paragraph{Families of payoff matrices.}
Aside from the \eqref{eq:MP} game,
our experimental evaluations of Fictitious Play variants
are performed on three high-dimensional families of
payoff matrices:
\begin{itemize}[
    leftmargin=2em,
    rightmargin=1em,
]
    \item   
    \textbf{Identity matrices:}
    Here, the payoff matrix is the $n\times n$
    identity matrix $I_n$ (i.e., the diagonal matrix 
    with diagonal entries all 1). 
    Recall that for standard FP,
    \cite{abernethy2021fast} established an $O(\sqrt{T})$ regret bound
    using fixed lexicographical tiebreaking.
    \item
    \textbf{Generalized Rock-Paper-Scissors (RPS) matrices:}
    Here, the payoff matrix is the $n\times n$
    generalization of the classic three-strategy Rock-Paper-Scissors game. 
    Specifically, $A$ is the matrix with entries $A_{i, j}$ 
    given by 
    \begin{equation}
        A_{i,j}
        \;:=\;
        \begin{cases}
          -1 &\text{if $j = i+1\; (\mod n)$}\\
         \;1 &\text{if $j = i-1\; (\mod n)$} \\
           0 &\text{otherwise}
        \end{cases}
        \qquad
        \text{for all $i, j \in [n]$} \;\;.
        \label{eq:RPS}
    \end{equation}
    For standard FP,~\cite{lazarsfeld2025fp} established
    an $O(\sqrt{T})$ regret bound for all such RPS matrices
    (using any tiebreaking rule), including when $A$ 
    is scaled by a constant, and when the non-zero entries
    have non-uniform weights. 
    \item 
    \textbf{Random {[}0,1{]} matrices:}
    We also consider $n\times n$ payoff matrices with 
    uniformly random entries in $[0, 1]$. 
    For these payoff matrices, there are no existing $O(\sqrt{T})$
    regret bounds for standard FP.
\end{itemize}

\paragraph{Tiebreaking rules.}
To evaluate the robustness of regret guarantees to 
the tiebreaking method, we run the FP variants
using  both (a) fixed \textit{lexicographical tiebreaking} 
(e.g., as in~\cite{abernethy2021fast})
and (b) uniformly \textit{random tiebreaking} 
(e.g., over the entries of the $\argmax$ set). 

\paragraph{Random initializations.}
To evaluate the robustness of regret guarantees to 
the players' initial strategies, we evaluated the Fictitious Play variants
over multiple random initializations of $x^0_1, x^0_2 \in \Delta_n$
(for the Alternating FP initialization from Figure~\ref{fig:regret-comp-intro},
note that the stated initialization is for $x^1_1 \in \Delta_n$,
as in the notation of Definition~\ref{def:alt-play}).
To generate a random initialization $x \in \Delta_n$, 
we sample $v \in [0, 1]^n$ with independent, uniformly random entries,
and normalize $x := v/\|v\|_1$.

\subsection{Empirical Regret Comparisons of 
Fictitious Play and Optimistic Fictitious Play}

\paragraph{Regret comparisons under randomized tiebreaking.}
In Table~\ref{table:comparison-01-lex} of Section~\ref{sec:conclusion}, 
we presented regret comparisons of Optimistic FP and standard FP
on the three families of payoff matrices described above
in Section~\ref{app:experiments:setup}, 
using \textit{fixed lexicographical tiebreaking}. 
In Table~\ref{table:comparison-02-rand}, we show the 
results of an identical experimental setup, now using
randomized tiebreaking. As in Table~\ref{table:comparison-01-lex}, 
the entries of Table~\ref{table:comparison-02-rand}
report average empirical regrets (and standard deviations)
over 100 random initializations, where for each initialization,
each algorithm was run for $T=10000$ iterations.

\begin{table}[h!]
 \centering
  \small
  \renewcommand{\arraystretch}{1.2}
  \setlength{\tabcolsep}{0.5em}
\begin{tabular}{c|| c c | c c | c c }
\textit{Dimension}:       & \multicolumn{2}{c | }{\textbf{15$\times$15}} & \multicolumn{2}{c| }{\textbf{25$\times$25}}   & \multicolumn{2}{c}{\textbf{50$\times$50}}   \\ 
\hline
Payoff Matrix     & \textbf{FP}           & \textbf{OFP}        & \textbf{FP}           & \textbf{OFP}         & \textbf{FP}           & \textbf{OFP} \\
\hline 
\textbf{Identity}           & 154.4 $\pm$ 4.2 & 8.4 $\pm$ 1.7 & 162.3 $\pm$ 3.4 & 12.9 $\pm$ 1.6 & 166.9 $\pm$ 2.2 & 25.0 $\pm$ 2.3 \\
\textbf{RPS}               & 235.2 $\pm$ 6.6 & 2.8 $\pm$ 0.5 & 241.5 $\pm$ 6.1 & 3.2 $\pm$ 0.9  & 242.9 $\pm$ 5.6 & 2.6 $\pm$ 0.8  \\
\textbf{Random {[}0, 1{]}} & 93.4 $\pm$ 5.0  & 2.7 $\pm$ 0.6 & 137.1 $\pm$ 6.1 & 7.0 $\pm$ 1.1  & 176.2 $\pm$ 6.3 & 12.2 $\pm$ 1.4
\end{tabular}
\vspace*{0.5em}
\caption{\small
    Empirical regret of FP and OFP using randomized tiebreaking.
    Each entry reports an average and standard deviation 
    (over 100 random initializations) of total regret after 
    $T = 10000$ steps. 
}
\vspace*{0em}
\label{table:comparison-02-rand}
\end{table}

As in Table~\ref{table:comparison-01-lex}, the results
of Table~\ref{table:comparison-02-rand} similarly show 
that Optimistic FP empirically obtains bounded
regret compared to the roughly $O(\sqrt{T})$ regret
of standard FP for each payoff matrix and dimension. 

\paragraph{Additional plots from fixed initializations.}
To further compare the empirical regrets of
standard FP and Optimistic FP, we present plots 
of the two algorithms run from fixed initializations,
similar to Figure~\ref{fig:regret-comp-intro} 
from Section~\ref{sec:intro} (which also included
a comparison with AFP). 
In each plot, we consider the three families of
identity, RPS, and random matrices
described earlier in Section~\ref{app:experiments:setup}. 
Note in particular that for the RPS game 
(including in Figure~\ref{fig:regret-comp-intro} of Section~\ref{sec:intro}),
for better visual comparison with the other games,
we use the payoff matrix specified in~\eqref{eq:RPS},
but scaled by the constant $2/3$. 

Figures~\ref{fig:regret-comp-02},~\ref{fig:regret-comp-03},
and~\ref{fig:regret-comp-04} show these comparisons
for $15\times15$ and $25\times25$ matrices,
using both randomized and lexicographical tiebreaking. 
In each instance, we again observe that Optimistic FP 
has bounded empirical regret compared to the $\sqrt{T}$
regret of standard FP. 

\begin{figure}[h!]
\centering
\includegraphics[width=0.95\textwidth]{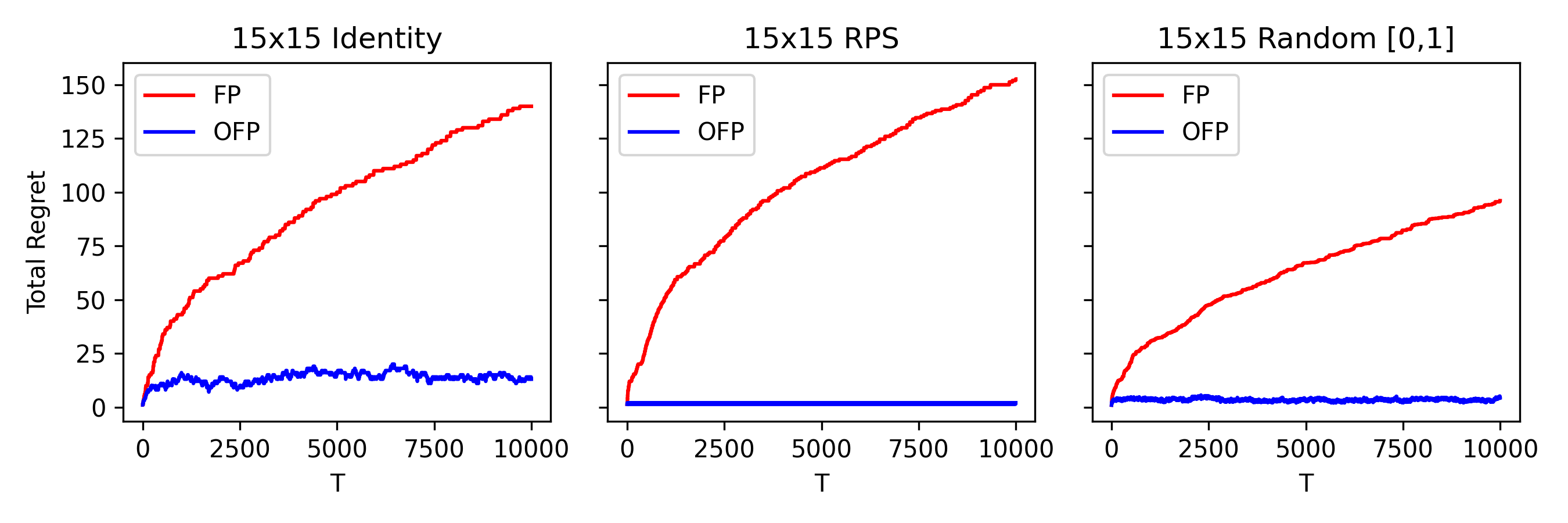}
\caption{
    \small
    Empirical regret of standard FP and Optimistic FP (OFP) 
    using \textbf{randomized tiebreaking} 
    on three $15\times15$ payoff matrices.
    For each payoff matrix, each algorithm was initialized from
    $x^0_1 = e_1, x^0_2 = e_n$ 
    and run for $T=10000$ iterations.
}   
\label{fig:regret-comp-02}
\end{figure}

\begin{figure}[h!]
\centering
\includegraphics[width=0.95\textwidth]{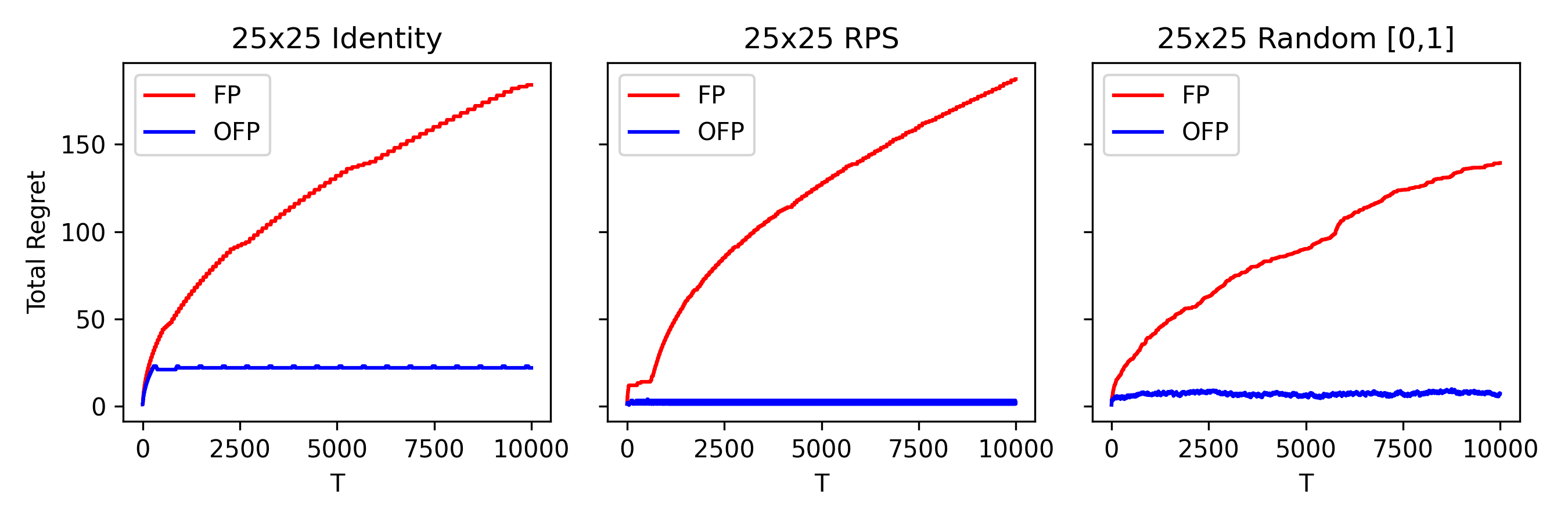}
\caption{
    \small
    Empirical regret of standard FP and Optimistic FP (OFP) 
    using \textbf{lexicographical tiebreaking} 
    on three $25\times25$ payoff matrices.
    For each payoff matrix, each algorithm was initialized from
    $x^0_1 = e_1, x^0_2 = e_n$ 
    and run for $T=10000$ iterations.
}
\label{fig:regret-comp-03}
\end{figure}

\begin{figure}[h!]
\centering
\includegraphics[width=1\textwidth]{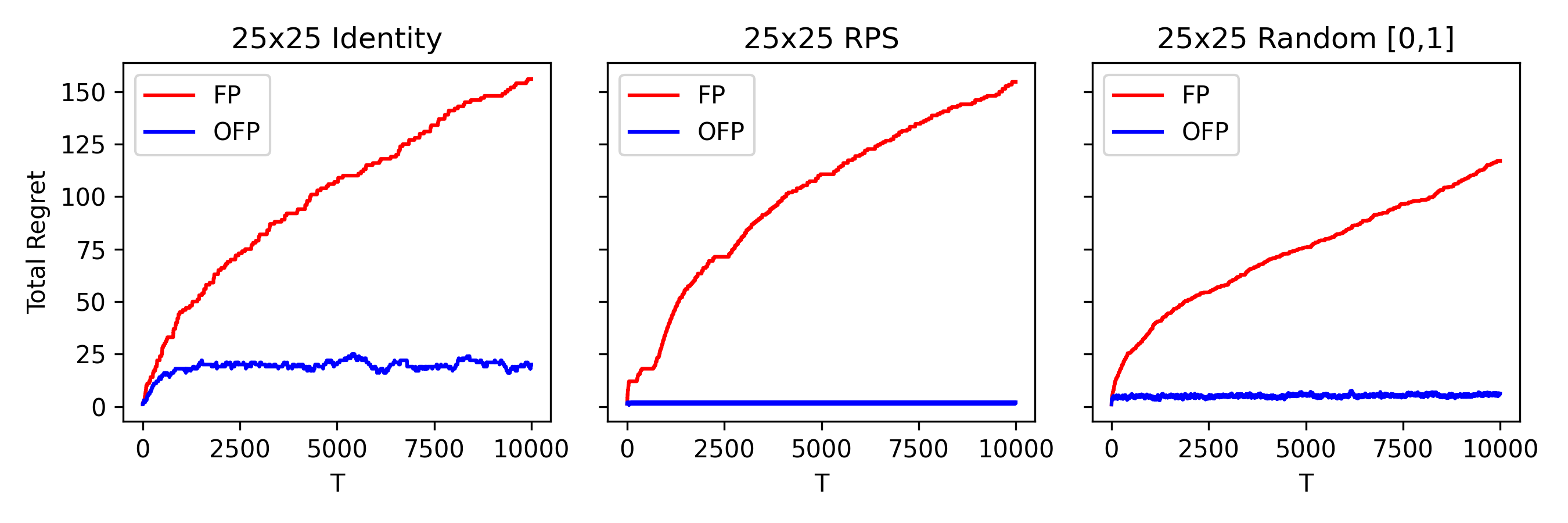}
\caption{
    \small
    Empirical regret of standard FP and Optimistic FP (OFP) 
    using \textbf{randomized tiebreaking} 
    on three $25\times25$ payoff matrices.
    For each payoff matrix, each algorithm was initialized from
    $x^0_1 = e_1, x^0_2 = e_n$ 
    and run for $T=10000$ iterations.
}
\label{fig:regret-comp-04}
\end{figure}

\end{document}